\documentclass{article} 
\usepackage{iclr2027_conference,times}

\usepackage{amsmath,amsfonts,bm}

\def\eqref#1{equation~\ref{#1}}

\def\1{\bm{1}}

\def\mA{{\bm{A}}}

\def\mD{{\bm{D}}}

\def\mI{{\bm{I}}}

\def\mL{{\bm{L}}}

\def\mV{{\bm{V}}}
\def\mW{{\bm{W}}}
\def\mX{{\bm{X}}}
\def\mY{{\bm{Y}}}
\def\mZ{{\bm{Z}}}

\DeclareMathAlphabet{\mathsfit}{\encodingdefault}{\sfdefault}{m}{sl}
\SetMathAlphabet{\mathsfit}{bold}{\encodingdefault}{\sfdefault}{bx}{n}

\def\gC{{\mathcal{C}}}
\def\gD{{\mathcal{D}}}
\def\gE{{\mathcal{E}}}
\def\gF{{\mathcal{F}}}
\def\gG{{\mathcal{G}}}
\def\gH{{\mathcal{H}}}

\def\gL{{\mathcal{L}}}

\def\gO{{\mathcal{O}}}

\def\gS{{\mathcal{S}}}

\def\gV{{\mathcal{V}}}
\def\gW{{\mathcal{W}}}

\newcommand{\R}{\mathbb{R}}

\usepackage{hyperref}
\usepackage{url}
\usepackage{amsmath,amssymb,amsfonts,amsthm}
\usepackage{graphicx}

\usepackage{booktabs}
\usepackage{todonotes}
\usepackage{cleveref}
\usepackage{adjustbox}
\usepackage{multirow}
\usepackage{wrapfig}

\newcommand{\model}{ONDA}
\newcommand{\modelextunderlined}{\underline{\textbf{o}}scillatory \underline{\textbf{n}}eural \underline{\textbf{d}}ynamics on she\underline{\textbf{a}}ves}
\newcommand{\modelext}{Oscillatory Neural Dynamics on Sheaves}
\newcommand{\ie}{i.e., }

\definecolor{first}{HTML}{009E73}  
\definecolor{second}{HTML}{D55E00} 
\definecolor{third}{HTML}{000000}  
\newcommand{\one}[1]{\textcolor{first}{\boldmath\bf#1}}
\newcommand{\two}[1]{\textcolor{second}{\boldmath\bf#1}}
\newcommand{\three}[1]{\textcolor{third}{\boldmath\bf#1}}

\newtheorem{theorem}{Theorem}
\newtheorem{proposition}{Proposition}
\newtheorem{corollary}{Corollary}
\newtheorem{lemma}{Lemma}
\title{Oscillatory Neural Dynamics over Sheaves}

\author{%
  \textbf{Jan-Willem Van Looy}\textsuperscript{1,}\thanks{Equal contribution. Contacts: Jan-Willem Van Looy: janwillem.vanlooy@unibo.it, Alessandro Trenta: alessandro.trenta@phd.unipi.it, Alessio Gravina: alessio.gravina@di.unipi.it, Alessio Borgi: alessio.borgi@uniroma1.it, ab3352@cam.ac.uk.}
  \quad \textbf{Alessandro Trenta}\textsuperscript{2,$*$}
  \quad \textbf{Alessio Gravina}\textsuperscript{2,$*$}
  \quad \textbf{Alessio Borgi}\textsuperscript{3,4$*$} \\
  \textbf{Ferdinando Zanchetta}\textsuperscript{1}
  \quad \textbf{Pietro Liò}\textsuperscript{4}
  \quad \textbf{Davide Bacciu}\textsuperscript{2}
  \quad \textbf{Rita Fioresi}\textsuperscript{1} \\[6pt]
  \textsuperscript{1}FaBiT, University of Bologna \\
  \textsuperscript{2}Department of Computer Science, University of Pisa \\
  \textsuperscript{3}Department of Computer, Control and Management Engineering, Sapienza University of Rome \\
  \textsuperscript{4}Department of Computer Science and Technology, University of Cambridge
}

\iclrfinalcopy 
\begin{document}

\maketitle
\lhead{Preprint}

\maketitle

\begin{abstract}
Effective long-range propagation remains a central challenge in graph neural networks, as increasing a model's propagation depth does not guarantee that distant nodes effectively influence each other. Sheaf neural networks enrich graph propagation through matrix-valued transport between stalks; still, this expressivity alone does not automatically imply effective long-range communication. We introduce \textsc{ONDA}, a long-range graph learning framework based on \emph{operator-valued information waves}. Stalk-valued representations evolve through second-order dynamics governed by learned sheaf transport operators, combining wave-like propagation with expressive local geometry. We characterize long-range influence through a stalk-wise sensitivity analysis and show that the cross-influence never vanishes. Across long-range propagation, severe graph bottlenecks, graph transfer, and heterophilic benchmarks, \textsc{ONDA} consistently improves over scalar wave propagation, diffusive sheaf baselines, and state-of-the-art models, demonstrating the benefit of coupling wave dynamics with matrix-valued transport.
\end{abstract}
\section{Introduction}\label{sec:introductuion}
Effectively modeling long-range dependencies remains a central challenge in deep learning for graphs \citep{shi2023exposition, gravina_gnn_ssm}. Graph Neural Networks (GNNs) \citep{nn4g, gnn, chebnet, MPNN, gcn} usually rely on local message-passing to exchange information between nodes, meaning that dependencies between non-adjacent regions must be mediated through repeated propagation steps. However, while increasing depth expands the receptive field (\ie the region from which information can be gathered), it does not guarantee that distant nodes communicate effectively. Indeed, challenges such as over-smoothing \citep{cai2020note, oono2020graph, rusch2023survey}, over-squashing \citep{alon2021oversquashing, topping2022understanding, diGiovanniOversquashing, mishayev2025shortrange}, and, more generally, vanishing gradients \citep{gravina_gnn_ssm} prevent relevant information exchange between distant nodes. For this reason, \emph{learning how to propagate information effectively over long graph distances} has become a critical problem in graph representation learning. 

Existing approaches address this problem from several directions. Rewiring methods modify the graph to shorten communication paths \citep{topping2022understanding,drew}, while graph Transformers allow direct message passing between any pair of nodes via attention mechanisms \citep{graphtransformer, graphgps}. A complementary line of work instead changes the dynamics of propagation while retaining sparse local interactions. In particular, differential-equation-inspired GNNs (DE-GNNs) have explored stable and non-dissipative dynamics as a way to maintain information over depth \citep{rusch2022graph, gravina_adgn,gravina_phdgn,gravina_stable_chebnet}. Oscillatory DE-GNNs \citep{rusch2022graph, gravina_sonar} follow this direction by modeling information propagation through second-order wave dynamics on the original graph.

While these approaches primarily focus on improving \emph{how far} information can propagate, another important aspect is \emph{how} information is exchanged between neighboring nodes. Many widely used GNNs generally propagate features through local aggregation operators based on scalar edge weights and shared feature transformations. Cellular sheaves provide a richer framework by associating local vector spaces (or stalks) with nodes and edges, together with linear restriction maps that relate node representations within a common edge space \citep{curry2014sheaves, hansen2019toward, hansen2020sheaf}. Sheaf Neural Networks (SNNs) learn these maps, inducing matrix-valued couplings between neighboring node stalks and substantially enlarging the class of graph-supported propagation operators \citep{bodnar2022neural}. This additional expressivity, however, does not by itself guarantee effective long-range communication between nodes. Indeed, neural sheaf models still rely on repeated dissipative local propagation, while information originating from distant regions must therefore traverse multiple successive sheaf transformations before reaching a target node. Consequently, maintaining informative signals over depth remains a largely underexplored challenge in SNNs, complementary to the problem of learning richer local transport.
Recent works have started investigating how to extend neural sheaf architectures through mechanisms like polynomial filters \citep{borgi2025polynomial}, 
large-scale diffusion \citep{bamberger2025bundle} and directional routing \citep{ribeiro2026cooperative} to broaden the radius of information propagation across the graph.

Rather than extending the spatial range of individual updates or controlling which messages traverse the graph, in this work, we seek to design a principled sheaf neural network whose dynamics enable useful stalk-valued information to remain active as it traverses the graph. This motivates \model{} (\modelextunderlined{}), a framework for long-range information propagation based on \emph{operator-valued information waves}. In \model{}, stalk-valued representations evolve through an oscillatory second-order dynamical system governed by a learned sheaf Laplacian. The learned sheaf shapes the propagation of these information waves: its restriction maps determine how information is transformed between stalks and how different propagated contributions reinforce or cancel when they interact. In this sense, the learned sheaf acts as a propagation medium whose local properties govern how information waves travel across the graph. This couples the persistence of wave propagation with the expressive local geometry of cellular sheaves, allowing information to remain influential over distance while being non-trivially transformed along the graph.

Our contributions are threefold: \textbf{(i)} we introduce \model{}, a new class of sheaf neural networks based on second-order oscillatory dynamics, providing a non-trivial extension of wave-like graph propagation to matrix-valued sheaf operators; \textbf{(ii)} we theoretically investigate the long-range propagation properties of \model{}, showing that its oscillatory dynamics preserve the influence of distant information; \textbf{(iii)} we demonstrate strong and consistent empirical performance across a diverse suite of tasks, including long-range propagation, over-squashing mitigation, and heterophilic graph learning.

\section{Related Work}
\label{sec:related}
\textbf{Long-Range Propagation on Graphs.} Although stacking message-passing layers progressively enlarges the receptive field, information from distant nodes can rapidly lose influence due to phenomena such as over-smoothing \citep{cai2020note, oono2020graph, rusch2023survey} and over-squashing \citep{alon2021oversquashing, topping2022understanding, diGiovanniOversquashing, mishayev2025shortrange}, both closely related to vanishing gradients \citep{gravina_gnn_ssm}. 
Various approaches aim to improve long-range information propagation. Graph rewiring methods \citep{topping2022understanding, DIGL, drew, effective_resistance} modify the graph topology to create more favorable communication paths, whereas Graph Transformers \citep{san, graphgps, graphtransformer, ying2021transformers, grit} use attention to learn direct interactions between nodes. Other methods modify the propagation dynamics through weight-space regularization \citep{gravina_adgn, gravina_swan}, port-Hamiltonian dynamics \citep{gravina_phdgn}, message filtering \citep{amp, finkelshtein2024cooperative}, 
or virtual nodes \citep{southern2025understanding}. Oscillatory GNNs \citep{rusch2022graph,gravina_sonar} exploit second-order oscillatory propagation dynamics to facilitate communication between distant nodes, but operate within conventional GNN frameworks and therefore rely on standard graph-based interactions. In contrast, \model{} introduces oscillatory dynamics into sheaf neural networks, combining persistent wave-like propagation with learned matrix-valued transport operators that can control how information is transformed and transmitted along individual edges. Together, these mechanisms allow \model{} to combine effective long-range propagation with richer, geometry-aware information transport. 

\textbf{Sheaf Neural Networks.} Cellular sheaves are used to extend graph-based learning by associating vector spaces with graph cells and linear restriction maps with their incidences, giving rise to matrix-valued generalizations of classical graph operators \citep{hansen2020sheaf}. Building on this framework, \cite{bodnar2022neural} introduced learnable sheaf structures into GNNs, showing that the resulting diffusion can provide favourable properties in heterophilic settings and alleviate over-smoothing. Subsequent works have explored different ways of parametrizing and learning the underlying sheaf, including connection-Laplacian constructions \citep{barbero2022sheaf}, polynomial filters \citep{borgi2025polynomial}, 
inductive biases inspired by joint diffusion processes \citep{hernandez-caralt2024joint}, and adaptive nonlinear sheaf Laplacians \citep{zaghen2024sheaf}. More recently, \cite{bamberger2025bundle} proposed using continuous diffusion over flat vector bundles to operate over larger spatial scales, while 
other approaches consider directional routing \citep{ribeiro2026cooperative} specifically designed to better exploit deep sheaf propagation. Nevertheless, most existing SNNs remain centered around diffusion-like dynamics, which are inherently dissipative and struggle to capture long-range interactions. In contrast, \model{} introduces second-order oscillatory dynamics into sheaf neural networks, interpreting learned matrix-valued restriction maps as the wave propagation medium whose local properties govern how information waves travel across the graph.

\section{\model{}}\label{sec:model}
We introduce \modelext{} (\model{}), a new class of sheaf neural networks driven by second-order oscillatory dynamics, allowing stalk-valued information to remain relevant as it traverses the graph. A schematic illustration of \model{} is provided in \Cref{fig:onda}.

\begin{figure}
    \centering
    \includegraphics[width=0.8\linewidth]{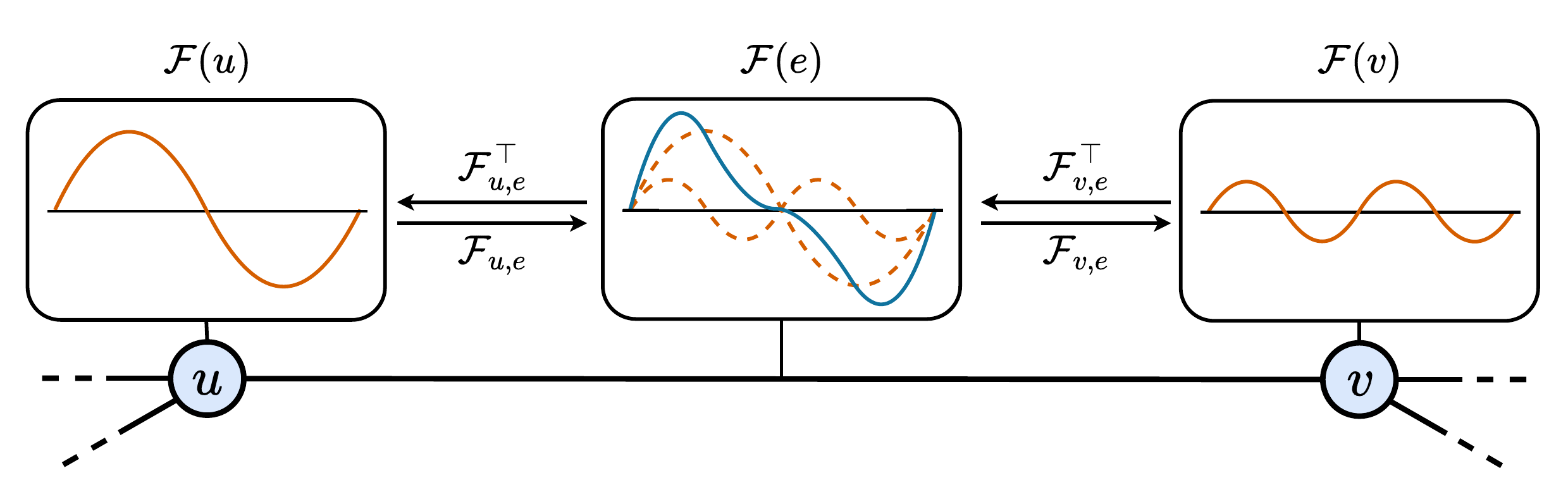}
    \caption{Schematic view of \model{} on a single edge. Orange curves show the oscillatory evolution of node states in the stalks \(\gF(u)\) and \(\gF(v)\); learned restriction maps express them in the shared edge stalk \(\gF(e)\) (dashed orange). The blue curve shows their superposition, highlighting constructive and destructive interference between the mapped signals.
    }
    \label{fig:onda}
\end{figure}
\subsection{Notation and preliminaries}
Let \(\gG  = (\gV , \gE )\) be an undirected graph, where \(\gV\) is the finite set of $n$ nodes (or vertices) and \(\gE \subseteq \mathcal{V} \times \mathcal{V}\) the set of edges. Given two nodes \(u,v \in \gV\) we denote an edge between them with the unordered pair \(\{u,v\}\). Graph connectivity is encoded by the adjacency matrix \(\mA\in\mathbb{R}^{n\times n}\), with degree matrix  \(\mD\in\mathbb{R}^{n\times n}\) defined by \(\mD_{ii} = \sum_{j=1}^n \mA_{ij}\). 
A \emph{cellular sheaf on \(\gG\)} \citep{shepard1985cellular,curry2014sheaves} assigns a finite-dimensional vector space \( \gF(v) \) to each vertex \( v \in \gV\), a finite-dimensional vector space \(\gF(e) \) to each edge \( e \in \gE\), and a linear restriction map
\( \gF_{ v, e}: \gF( v) \to \gF(e)\) to each incidence \( v \in e\). We refer to \( \gF(v) \) and \( \gF(e) \) as the \emph{vertex} and \emph{edge stalks}, respectively. In this paper, all stalks share a common dimension \(d\geq 1\). The space of \emph{0-cochains}, or \emph{stalk-valued node states} is \( C^0( \gG  ; \gF ) := \bigoplus_{v\in \gV } \gF(v).\) We write \(\mX=(\mX_v)_{v\in \gV }\in C^0(\gG ; \gF) \), where \( \mX_v \in \gF(v)\). For \( c \) feature channels, node features lie in \( C^0(\gG ; \gF )\otimes\mathbb{R}^{c}\). The \emph{Sheaf Laplacian} \citep{hansen2019toward, hansen2020laplacians} \( L_{ \gF }: C^0(\gG ; \gF ) \to C^0(\gG ; \gF )\) acts nodewise as:
\begin{equation}
\label{eq:sheaf-laplacian-nodewise}
(L_{\gF}\mX)_v = \sum_{\substack{e=\{v,u\}\in \gE }} \gF_{v,e}^{*} \bigl( \gF_{v,e} \mX_v - \gF_{u,e}\mX_u \bigr) \end{equation} 
where each stalk is equipped with an inner product and \( \gF_{v,e}^{*} \) denotes the adjoint of \( \gF_{v,e} \). The resulting sheaf Laplacian is self-adjoint and positive semidefinite \citep{hansen2019toward}. Along each edge, the restriction maps place neighbouring node representations in a common space, where their discrepancy is computed; the adjoint then maps this discrepancy back to the incident node, where it is aggregated with the remaining neighbour contributions. For \(d=1\), Euclidean inner products, and identity restriction maps, \(L_{\gF}\) reduces to the standard graph Laplacian \(\mL_\gG =  \mD - \mA\) \citep{brouwer2011spectra}. We also use the symmetrically normalized sheaf Laplacian \( \Delta_{\gF} := \gD_{\gF}^{ -1/2} L_{\gF} \gD_{\gF}^{-1/2},\) where \( \gD_{\gF} \) is the block diagonal part of \( L_{\gF} \) \citep{hansen2019toward}.

Given its oscillatory dynamics, \model{} is inherently a differential-equation-based architecture, where node representations are interpreted as the evolving states of a dynamical system defined over the graph \citep{GDE, gravina_adgn, han2024from}. Under this view, we denote the collection of all node states at time $t$ as \( \mX(t) \in C^0(\gG ;\gF) \otimes \mathbb R^c\), stacked as an \( nd \times c\) matrix. 
As in DE-GNNs, the continuous dynamics are instantiated as a neural architecture through numerical discretization \citep{gravina_adgn}. 
Therefore, given an initial condition \(\mX(0) = \overline{\mX}\) and a step size $h>0$, we denote by \( \mX^{ (\ell) }\) the node features after \(\ell\) integration steps, initialized with \( \mX^{(0)} = \overline{\mX} \). Thus, \(\mX^{(\ell)} \approx\mX(t_\ell)\), where \(t_\ell=h\ell\)
and the parenthesized superscript indexes the discrete steps.

\subsection{\modelext{}}\label{subsection:onda_modl_descr}
Neural sheaf models replace scalar graph interactions with matrix-valued transport between local stalks, transforming information as it crosses each edge. Yet, expressive transport alone does not guarantee that information remains influential over long propagation ranges. \model{} addresses both aspects jointly: the learned sheaf controls how information moves between neighbouring stalks, while second-order oscillatory dynamics govern how it evolves with depth.

Specifically, we model the evolution of the stalk-valued representations \( \mX(t) \in C^0(\gG ;\gF ) \otimes \mathbb R^c \) through the following second-order differential equation
\begin{equation}\label{eq:onda_sheaf_wave}
    \ddot{\mX}(t) = -L_{\gF}(\mI_n\otimes\mW_1)\mX(t)\mW_2 -R_\theta(\mX(t))\odot\dot{\mX}(t)+F_\theta(\mX(t)), \quad {\mX}(0) = \overline{\mX}
\end{equation}
where \( L_{\gF} \) is the sheaf Laplacian from \Cref{eq:sheaf-laplacian-nodewise}, \(\mI_n \in \mathbb{R}^{n\times n}\) is the identity matrix, \(\mW_1\in\mathbb{R}^{d\times d}\) is the weight matrix acting on stalk coordinates, and \(\mW_2\in\mathbb{R}^{c\times c}\) transforms feature channels. The terms \(R_\theta(\mX(t))\in\mathbb R_{\geq0}^{nd\times c}\) and \(F_\theta(\mX(t))\in \mathbb R^{nd\times c}\) provide state-dependent dissipation and external forcing, respectively, while $\odot$ denotes the  Hadamard product. 

In \Cref{eq:onda_sheaf_wave}, the sheaf operator shapes the information waves: restriction maps determine how signals transform across edges and how incoming wave contributions combine within each stalk. This mirrors classical wave equations \citep{Evans2010-PDEs}, where the spatial Laplacian operator encoded the propagation properties of the underlying medium. \model{} therefore couples the propagation bias of wave dynamics with learned matrix-valued transport.

With zero dissipation and forcing and identity weight matrices, \Cref{eq:onda_sheaf_wave} reduces to a cellular-sheaf generalization of the graph wave equation \( \ddot{\mX}(t) = -\mL_{\gG} \mX(t)\). Here, \( \mL_{\gG} \) is the graph Laplacian which governs spatial coupling in the graph wave equation and induces conservative oscillatory dynamics.

Purely conservative wave dynamics can nevertheless be restrictive for downstream learning, because not all information should necessarily be preserved indefinitely. We therefore augment the oscillatory term with the state-dependent dissipation \(R_\theta(\mX(t))\) and external forcing \(F_\theta(\mX(t))\). Dissipation adaptively attenuates propagated information, while forcing supplies an additional learned drive. Together, they let \model{} balance conservative and non-conservative behavior and adapt its propagation dynamics to the task at hand.

\paragraph{Discretization and Implementation Details.}\label{par:discr_impl_detail} We obtain a neural architecture from \Cref{eq:onda_sheaf_wave} by numerically discretizing its continuous dynamics. Introducing the velocity \(\mV(t)=\dot{\mX}(t)\) and applying a finite-difference scheme \citep{Ferziger2001-PDENumerical} yields the update 
\begin{equation}
\begin{aligned}
\mV^{( \ell+1)} & = \mV^{(\ell)}-hL_{\gF} \bigl( ( \mI_n \otimes \mW_1) \mX^{(\ell)} \mW_2\bigr)+hF_\theta(\mX^{(\ell)}) -h R_\theta( \mX^{(\ell)})\odot\mV^{(\ell)}, \\
\mX^{(\ell+1)} & = \mX^{(\ell)}+h\mV^{(\ell+1)}
\end{aligned}
\label{eq:onda_discrete_updates}
\end{equation}
where \(h>0\) is the integration step size, and \( \ell=0, \ldots,\gL-1\).

To increase expressivity, we stack multiple instances of the propagation dynamics in \Cref{eq:onda_discrete_updates} and interleave them with nonlinear feature transformations, following the general principle of separating recurrent propagation from channel-wise nonlinear mixing
\citep{gu2021efficiently,orvieto2023resurrecting,mpssm,gravina_sonar}. A \emph{block} consists of \(\gL\) propagation steps from \Cref{eq:onda_discrete_updates}, followed by an MLP transformation. In block \(b\), the features after \(\gL \) propagation steps initialize block \( (b+1) \) through:
\begin{equation}
\mX^{(b+1,0)} = \operatorname{MLP}_b(\mX^{(b, \gL)}),
\qquad b=0, \ldots, B-1
\label{eq:onda_block_composition}
\end{equation}
The velocity of the new block is again initialized from its input
representation and the final block output is passed to a task-specific decoder. Propagation within each block can be either fixed or adaptive. In the fixed setting, restriction maps are inferred once from the block input and held constant for all \( \gL \) steps: each block first learns the sheaf governing transport, then evolves node representations according to \Cref{eq:onda_discrete_updates}.
In the adaptive setting, restriction maps are recomputed at each step from the current features, allowing the evolving node states to reshape the propagation medium. 
The number of blocks \( B \), the number of propagation steps \( \gL \), and the integration step size \( h \) are treated as architectural hyperparameters.

Following standard SNN practice \citep{bodnar2022neural,bamberger2025bundle,ribeiro2026cooperative}, we parametrize restriction maps with neural networks. We consider both symmetrically normalized sheaf Laplacians and a directed variant. In the latter, restriction maps are learned independently for each edge orientation, and each discrepancy is applied only at the receiving endpoint. For simplicity, we retain the notation \(L_{\gF}\). Further details on sheaf normalization and the directed variant are provided in Appendix~\ref{subsec:app_sheaf_lap}; the restriction-map parametrizations are described in Appendix~\ref{app:restriction_parameterizations}.

\section{Theoretical Analysis}\label{sec:analysis}
In this section, we study the evolution of node states in a single \model{} block, with particular focus on providing theoretical grounding for long-range influence across nodes. We start from the continuous definition of the model in \Cref{eq:onda_sheaf_wave} without dissipation and external forcing, showing that the oscillatory dynamics provide explicit guarantees on the cross-influence between nodes. In Theorem \ref{prop:discrete_source_target_response}, we translate these results to the discrete case. 

To isolate the benefits of the oscillatory dynamics, we fix the restriction maps within one block and set the weights to the identity. For simplicity, we also work with a single feature channel and write \( \gH:= C^0 ( \gG; \gF ) = \bigoplus_{v \in \gV} \gF(v) \cong \R^{n d} \) for the feature space. With a fixed sheaf Laplacian \( L_\gF : \gH\rightarrow \gH\), the conservative form of \model{} is given by:
\begin{equation}
\ddot{\mX}(t) + L_\gF \mX(t) = 0, \qquad \mX(0) = \overline{\mX}, \qquad \dot{\mX}(0) = \overline{\mV}
\label{eq:fixed_operator_wave}
\end{equation}
where $\overline{\mX}$ and $\overline{\mV}$ are respectively the initial node features and their velocities. We now define the equivalent of the sine and cosine functions for the sheaf Laplacian through the corresponding power series expansion \citep{higham2008functions}:
\begin{equation}
    \gC_{\gF}(t) = \sum_{k=0}^{\infty} (-1)^k \frac{t^{2k}}{(2k)!} L_\gF^k, \qquad   \gS_{\gF}(t) = \sum_{k=0}^{\infty} (-1)^k \frac{t^{2k+1}}{(2k+1)!} L_\gF^k
\end{equation} 
Analogously to the wave equation in physics \citep{Evans2010-PDEs}, the general solution to \Cref{eq:fixed_operator_wave} is given by \(\mX(t) = \gC_{\gF}(t)\overline{\mX} + \gS_{\gF}(t) \overline{\mV} \), as shown in Appendix~\ref{subsec:app_sol_wave_eq}. This describes the node features at time \(t \) through a combination of the sine and cosine operators.

We now seek to measure how different nodes influence each other through the common framework of sensitivity analysis \citep{diGiovanniOversquashing}
. Concretely, we measure how changes in the features of node $u$ affect a target node $v$. For each \( u \in\gV \), we define \( \iota_u: \gF(u) \to \gH\) and \(\pi_u:  \gH \to \gF(u)\) as the inclusion and projection operators, respectively. We define $\mW\in \gF(u)$ as a perturbation of the initial state at node $u$. Thus, the perturbed initial condition of the whole graph is $\overline{\mX} + \iota_u \mW$, where the components relative to $u$ are shifted by a vector \( \mW \). Given the closed-form solution of \Cref{eq:fixed_operator_wave}, we can calculate the difference between the perturbed and unperturbed trajectories as:
\begin{equation}\label{eq:trajectory_change_perturbation}
    \Delta\mX(t) = (\gC_{\gF}(t)(\overline{\mX} + \iota_u \mW) + \gS_{\gF}(t)\overline{\mV}) - (\gC_{\gF}(t)\overline{\mX} + \gS_{\gF}(t)\overline{\mV}) = \gC_{\gF}(t)\iota_u\mW
\end{equation}

This equation already shows that changes in the stalk of a single node $u$ affect the whole trajectory indefinitely, as $\gC_{\gF}(t)$ does not vanish in the limit of $t \rightarrow \infty $. We refer to the evolving response \( \Delta\mX(t) \) generated by a perturbation at node \(u\) as its \emph{information wave}.  Projecting this wave onto a target node $v$ gives $\Delta \mX_v(t) = \pi_v\gC_{\gF}(t)\iota_u \mW$. Independent of the particular perturbation \(\mW\), we therefore define the \emph{source-to-target influence} at time \( t \) as:
\begin{equation}
J^\gF_{ v \leftarrow u}(t) := \pi_v \gC_{\gF}(t) \iota_u
= \bigl( \gC_{\gF}(t) \bigr)_{vu} : \gF(u)\rightarrow\gF(v)
\label{eq:stalk_sensitivity}
\end{equation}
In contrast to the scalar case, the map \( J^\gF_{v\leftarrow u}(t)\) is a matrix-valued transformation between the source and target stalks, describing both the strength of the interaction and how directions in \(\gF(u)\) are transformed in \(\gF(v) \). The fact that the operators $\gC_{ \gF }, \iota_u, \pi_v$ are all linear allows for the linear superposition of waves and their effects, similar to how physical waves linearly compose and decompose \citep{Evans2010-PDEs}.
In the small perturbation limit, the source-to-target response in \Cref{eq:stalk_sensitivity} can be thought of as the equivalent of the Jacobian matrix in sensitivity analysis on standard GNNs \citep{diGiovanniOversquashing} to the sheaf case.

To understand whether the non-vanishing influence of node \( u \) on the whole trajectory enables long-range propagation, we first examine the behaviour of \Cref{eq:stalk_sensitivity} depending on the target node $v$ in terms of the graph structure. Since \( (L_{\gF})_{vu} = 0\) whenever distinct nodes \(u,v\) are not adjacent, applying \( L_{\gF} \) once only couples the stalk of node $u$ to itself or adjacent ones. Subsequent powers \(L_{\gF}^k\) in the definition of the cosine operator $\gC_{\gF}$ can couple only nodes connected by a path of length at most \(k\). For \( k \geq 0\), we formalise this idea (see also Appendix \ref{subsec:proof_pathwise_source_target_response}) by defining \(\gW_k(u,v)\) as the set of length-\(k\) walks \( \omega=(v_0 = u, v_1, \ldots, v_k = v)\), where each step either follows an edge or remains at the
current node. From this we can construct an operator $ (L_{\gF})_{\omega}$ as the ordered composition of Laplacian blocks along $\omega$, that is, \( (L_{\gF})_\omega := (L_{\gF})_{v_kv_{k-1}} \ldots (L_{\gF})_{v_1v_0}: \gF(u) \rightarrow \gF(v).\)

 For \( k = 0 \), we define the empty product as the identity on \(\gF(u)\). In particular, this operator describes the stalk-to-stalk transformation accumulated along the walk \(\omega\) and lets us decompose the source-to-target response across all different walks that connect $u$ to $v$. 
\begin{proposition}[Pathwise source-to-target response] \label{prop:pathwise_source_target_response} Let \( u, v \in \gV \) be two nodes connected by a shortest path of length \(r\). 
Then:
\begin{equation}
J^\gF_{ v \leftarrow u}(t) = \sum_{ k=r }^{\infty} (-1)^k\frac{t^{2k}}{(2k)!} \sum_{\omega\in\gW_k(u,v)}(L_{\gF})_\omega \label{eq:pathwise_source_target_response}
\end{equation}
If there is no path connecting \(u\) and \(v\) in the graph, then \( J^\gF_{v\leftarrow u}(t)=0\) for every \(t\).
\end{proposition}

The proof is provided in Appendix~\ref{subsec:proof_pathwise_source_target_response}. Thus, if \( u \) and \( v \) are \( r \) edges apart, the first possible contribution comes from shortest paths of length \(r\), with longer walks contributing through higher powers of \( L_{\gF } \). A key observation is that the propagation time \( t \) weighs the contribution of different walks through the coefficients \( \frac{( -1 )^k t^{2k}} {(2k)!} \). Hence, the contribution of walks of different lengths to the source-to-target response depends on the propagation time \( t \), and graph distance does not prevent communication between connected nodes. At early times, shorter walks dominate the response, while increasing \( t \) allows contributions from progressively longer walks to become relevant. Intuitively, the longer the wave evolves, the more opportunities it has to propagate through the graph and influence the target through different routes.

While Proposition \ref{prop:pathwise_source_target_response} describes how the source-to-target response is assembled from local interactions, we now describe its evolution as time varies. For the self-adjoint sheaf Laplacian, this temporal behaviour can be characterised through its spectral theory \citep{hansen2019toward, hansen2020laplacians}. By the finite-dimensional spectral theorem \citep{Halmos1974}, we can decompose \( L_{\gF} = \sum_\lambda \lambda P_\lambda\), where the sum runs over distinct eigenvalues \( \lambda \geq 0\) and \( P_\lambda \) is the orthogonal projector onto the corresponding eigenspace. As before, we write \( (P_\lambda)_{ vu} :=  \pi_v P_\lambda \iota_u : \gF(u) \rightarrow \gF(v)\) for the associated source-to-target block.

\begin{proposition}[Spectral source-to-target response] \label{prop:spectral_stalk_sensitivity} For the sheaf Laplacian \(L_{\gF}\), the source-to-target sensitivity satisfies:
\begin{equation}\label{eq:spectral_stalk_sensitivity}
    J^{\gF}_{v \leftarrow u}(t) =  (P_0)_{vu} + \sum_{\lambda > 0} \cos(t \sqrt{\lambda})(P_\lambda)_{vu}
\end{equation}
\end{proposition}

The proof is provided in Appendix~\ref{app:proof_prop2}. This statement makes the oscillatory nature of the response explicit. Each block \((P_\lambda)_{vu}\) gives the contribution of the corresponding wave frequency and eigenspace, while the factor \( \cos( t \sqrt{\lambda} )\) determines how that contribution changes over time. The term \( (P_0)_{vu} \) remains constant, whereas the positive-eigenvalue contributions oscillate. Waves at different frequencies can reinforce or cancel, so a non-trivial source-to-target response can be weak or strong depending on time, but cannot decay to zero. This fundamentally distinguishes \model{} from sheaf models based on heat-flow dynamics, where positive-eigenvalue components decay exponentially \citep{bodnar2022neural, bamberger2025bundle}.

The continuous-time analysis shows that injecting oscillatory biases into feature dynamics sustains long-range influence on sheaves. We now explain how the above results directly translate into the discrete updates from \Cref{eq:onda_discrete_updates}, under the same assumptions of no dissipation or external forcing. Starting from the same perturbation \(\mW \in \gF(u)\) at the source node \( u \), let: \begin{equation}
    \Delta \mX^{(\ell)} := \mX^{(\ell)}(\overline{\mX} + \iota_u\mW) - \mX^{(\ell)}(\overline{\mX}) 
\end{equation} denote the resulting change in the node features after \( \ell \) discrete updates. Its component at a target node \( v \) is \( \Delta\mX_{v, \ell} := \pi_v \Delta \mX^{(\ell)} \in \gF(v). \) As in the continuous case, we define the \emph{discrete source-to-target sensitivity} \( J^{ \gF}_{ v \leftarrow u}[\ell,h] \) by the linear map satisfying \( \Delta \mX_{v,\ell} = J^{\gF}_{ v\leftarrow u } [\ell,h] \mW.\) Thus, \( J^{ \gF}_{v\leftarrow u}[\ell, h] \) describes how a perturbation in the source stalk is transformed into a response at the target after \( \ell \) discrete updates with step size \( h \). The following result gives this discrete source-to-target sensitivity a closed-form expression in terms of the spectral projectors, under a mild stability condition. In particular, it applies to the symmetrically normalized sheaf Laplacian \( \Delta_{\gF} \) for \(0 < h < \sqrt{2}\), since its spectrum is contained in \([0,2]\), as shown in Appendix~\ref{subsec:app_sheaf_lap}.

\begin{theorem}[Discrete source-to-target response]
\label{prop:discrete_source_target_response} Suppose \( L_{\gF} \) is self-adjoint and positive semidefinite and \( h^2 \lambda_{\max} (L_{\gF}) < 4. \) Then, for every integer \( \ell \geq 0 \):
\begin{equation}
    J^{\gF}_{v \leftarrow u }[\ell,h]  = \sum_{\lambda}
    \frac{ \cos \bigl( ( \ell + \tfrac12)\theta_h(\lambda) \bigr) }{ \cos \bigl ( \theta_h( \lambda)/2\bigr) }
    ( P_\lambda )_{vu}
\label{eq:discrete_wave_spectral_sensitivity}
\end{equation}
where \( \theta_h (\lambda) = 2\arcsin \! \left( \frac{h \sqrt{\lambda}}{2} \right). \) Moreover, \( J^{\gF}_{v \leftarrow u}[\ell ,h] = 0 \) whenever there is no path from \(u\) to \(v\) of length at most \(\ell\). Finally, for each node pair $u,v$, the response is bounded and, if it is non-zero at some step, it does not converge to zero as \(\ell \to \infty\).
\end{theorem}

The proof is provided in Appendix~\ref{subsec:app_proof_th1}. Thus, both features and properties on long-range propagation and non-vanishing of responses we discussed previously still hold for the discretized dynamics that define \model{}. First, the local nature of message-passing updates enables communication between nodes at distance $r$ only after $\ell \geq r$ iterations. At the same time, each positive-eigenvalue contribution remains oscillatory across depth and time, and does not vanish, making the discrete dynamics consistent with the continuous ones. Indeed, as shown in Appendix~\ref{subsec:app_proof_th1}, for each fixed eigenvalue \(\lambda\), the coefficient multiplying \((P_\lambda)_{vu}\) in \Cref{eq:discrete_wave_spectral_sensitivity} converges to \(\cos(T\sqrt{\lambda})\) as \(\ell \to \infty\) and \(h\to0\) with \(\ell h=T\) fixed. Taking this limit in the spectral expansion therefore recovers the continuous-time expression at time \(T\). Thus, for sufficiently small step sizes, \model{} inherits the persistent oscillatory behaviour of the continuous dynamics, leveraging matrix-valued sheaf operators for long-range propagation.

\section{Experiments}\label{sec:experiments}
We evaluate \model{} across complementary settings that isolate distinct aspects of effective information propagation on graphs. \Cref{sec:echo} examines long-range propagation across structurally challenging topologies with severe information bottlenecks. \Cref{sec:barbell} provides a controlled test of short-range over-squashing, requiring densely connected regions to exchange aggregated information through a narrow communication channel. \Cref{sec:experiments/transfer} evaluates transfer across varying source-to-target distances and graph topologies. We further test \model{} beyond explicitly propagation-oriented tasks on heterophilic benchmarks in Appendix~\ref{sec:heterophilic}. Appendix \ref{sec:experimental-details} details the sheaf learners and hyperparameter search space, Appendix ~\ref{app:barbell_full_results} reports an ablation of the restriction-map parametrization, and Appendix \ref{app:complexity} analyzes computational cost and runtime.

\subsection{Long-range Propagation: ECHO Benchmark}\label{sec:echo}
We first evaluate \model{} on ECHO \citep{gravina_echo_benchmark}, a benchmark specifically designed to assess long-range propagation capabilities. Its synthetic suite, ECHO-Synth, comprises single-source shortest path (SSSP), node eccentricity, and graph diameter prediction, tasks whose solutions depend on information travelling across distant regions of the graph. Models with limited propagation range therefore incur high error. We follow data splits, hyperparameters search space, and experimental protocol of \citep{gravina_echo_benchmark}.

\textbf{Results.}
As shown in \Cref{tab:results_echo}, \model{} achieves the lowest error on both SSSP and eccentricity, substantially outperforming the strongest competing methods, while remaining competitive on diameter prediction. It also consistently surpasses SONAR and GraphCON, the two most closely related oscillatory GNN baselines, as well as established sheaf models, including NSD and CSNN. These results show that \model{}'s oscillatory sheaf dynamics are particularly effective on tasks designed to stress long-range information propagation.
\begin{table}[t]
\caption{\textbf{ECHO-Synth benchmark results
.} Performance comparison on MAE ($\downarrow$). Baseline values are reported directly from \cite{gravina_echo_benchmark}. For each column, the \one{best}, \two{second-best}, and \three{third-best} results are highlighted.} 
\label{tab:results_echo}
\centering

\setlength{\tabcolsep}{4pt}
\vspace{1mm}
\small
\begin{tabular}{l ccc cc}
\toprule
\textbf{Model}  & \textbf{sssp} $\downarrow$ & \textbf{ecc} $\downarrow$   & \textbf{diam} $\downarrow$  \\
\midrule
A-DGN           & $1.176_{\pm 0.140}$        & $4.981_{\pm 0.037}$         & \three{$1.151_{\pm 0.038}$} \\
DRew            & $1.279_{\pm 0.011}$        & $4.651_{\pm 0.020}$         & $1.243_{\pm 0.047}$         \\
GCN             & $2.102_{\pm 0.094}$        & $5.233_{\pm 0.034}$         & $3.832_{\pm 0.262}$         \\
GCNII           & $2.128_{\pm 0.429}$        & $5.241_{\pm 0.030}$         & $2.005_{\pm 0.093}$         \\
GIN/GINE        & $2.234_{\pm 0.271}$        & $4.869_{\pm 0.092}$         & $1.630_{\pm 0.161}$         \\
GPS             & $0.472_{\pm 0.050}$        & $4.758_{\pm 0.021}$         & $2.160_{\pm 0.098}$         \\
GraphCON        & $5.734_{\pm 0.011}$        & $5.474_{\pm 0.001}$         & $2.969_{\pm 0.189}$         \\
GRIT            & \two{$0.121_{\pm 0.013}$}  & $5.091_{\pm 0.158}$         & \one{$1.014_{\pm 0.046}$}   \\
PH-DGN          & $1.323_{\pm 0.485}$        & $5.068_{\pm 0.126}$         & $1.627_{\pm 0.398}$         \\
SWAN            & $0.896_{\pm 0.232}$        & $4.840_{\pm 0.045}$         & \two{$1.121_{\pm 0.070}$}   \\
SONAR           & $0.527_{\pm 0.254}$         & \two{$2.393_{\pm 0.808}$}   & $1.239_{\pm 0.151}$         \\
\midrule
NSD             & \three{$0.231_{\pm0.038}$}  & $4.672_{\pm0.315}$          & $1.448_{\pm0.097}$          \\
CSNN            & $0.421_{\pm0.165}$          & \three{$4.522_{\pm0.048}$}  & \three{$1.134_{\pm0.094}$}  \\
\midrule
\textbf{\model{}}  & \one{$0.085_{\pm0.019}$} & \one{$0.885_{\pm 0.158}$}   & $1.183_{\pm0.020}$          \\
\bottomrule
\end{tabular}
\end{table}

\subsection{Graph Bottleneck Information Transfer: Barbell Benchmark} \label{sec:barbell}
Following \citep{bamberger2025bundle}, we consider the Barbell benchmark to assess robustness to over-squashing. This synthetic node-regression task consists of two cliques connected by a single bridge edge and requires information to be transferred between them. Node features are independently sampled from two disjoint distributions, one per clique, and each node must predict the empirical feature mean of the opposite subgraph. Despite the graph diameter being three, every cross-clique signal must traverse the same bridge edge. The benchmark therefore isolates whether a model can aggregate a clique-level statistic and transport it through a severe communication bottleneck.
We report nodewise mean-squared error, averaged over test graphs. Under the standard benchmark scaling, an MSE near \(1\) indicates no effective cross-clique transfer, values between \(0.4\) and \(0.6\) indicate partial transfer, and values below \(0.25\) are considered to solve the task \citep{bamberger2025bundle, gravina_stable_chebnet}.

\noindent\textbf{Results.}
Table~\ref{tab:barbell-results} summarizes Barbell results, with an extended comparison in Appendix~\ref{app:barbell_full_results}.
Standard NSD variants fail to achieve effective cross-clique transfer at any tested graph sizes, underscoring the need for an explicit propagation mechanism inn SNNs. Even CSNN and SONAR, the closest competitors from the sheaf and oscillatory perspectives, respectively, achieve only partial transfer. By contrast, \model{} solves the task near-perfectly at \( N=10 \) and \( N=20\). The gap is most revealing at \( N=50\), where \model{} is the only method to satisfy the solving criterion, while all baselines exhibit only limited information transfer. 

\begin{table}[t]
\centering
\caption{\textbf{Barbell-clique Results.
}
Test MSE ($\times 10^{-3}$; $\downarrow$) 
Results are reported as mean$_{\pm\mathrm{std}}$ over 5 seeds using the
best configuration selected by validation. For each column, the \one{best}, \two{second-best}, and \three{third-best} results are highlighted.}
\label{tab:barbell-results-full}
\small
\vspace{1mm}
\setlength{\tabcolsep}{6pt}
\begin{tabular}{lccc}
\toprule
\textbf{Model}
& $\boldsymbol{N=10}$ $\downarrow$
& $\boldsymbol{N=20}$ $\downarrow$
& $\boldsymbol{N=50}$ $\downarrow$ \\
\midrule

GCN
& $884.200_{\pm 6.900}$
& $961.100_{\pm 12.300}$
& $860.000_{\pm 35.000}$ \\

GAT
& \two{$7.900_{\pm 2.100}$}
& \two{$487.000_{\pm 342.200}$}
& $852.000_{\pm 29.000}$ \\

GraphSAGE
& $455.500_{\pm 199.600}$
& $998.200_{\pm 6.500}$
& $876.000_{\pm 32.000}$ \\

GIN
& $900.500_{\pm 47.900}$
& $970.600_{\pm 3.900}$
& $985.000_{\pm 123.000}$ \\

GraphTransformer
& $670.300_{\pm 306.800}$
& $1041.200_{\pm 46.900}$
& $831.000_{\pm 13.000}$ \\

ChebNet
& \three{$10.600_{\pm 1.200}$}
& $983.600_{\pm 5.100}$
& $922.000_{\pm 47.000}$ \\

SONAR
& $462.800_{\pm 250.300}$
& \three{$717.500_{\pm 249.600}$}
& \three{$809.000_{\pm 68.000}$} \\

\midrule

NSD
& $937.200_{\pm 11.400}$
& $1005.100_{\pm 7.800}$
& $825.000_{\pm 32.000}$ \\

CSNN
& $473.900_{\pm 106.200}$
& $890.800_{\pm 265.400}$
& \two{$742.000_{\pm 168.000}$} \\

\midrule

\textbf{\model{}}
& \one{$0.700_{\pm 0.300}$}
& \one{$0.900_{\pm 0.500}$}
& \one{$83.000_{\pm 85.000}$} \\

\bottomrule
\end{tabular}
\end{table}

\subsection{Graph Transfer Task} \label{sec:experiments/transfer}
We consider the graph transfer task introduced by \citep{diGiovanniOversquashing}, following the regression setup of \citep{gravina_swan}, with the aim of assessing model capabilities to propagate information at increasing distances with different topologies and bottlenecks. Nodes are initialized with random features, except for the source and target nodes, whose values are set to \(1\) and \(0\), respectively. These nodes are separated by \(\kappa \in\{3,5,10,50\}\) hops, and the objective is to swap their values while leaving all other node features unchanged. The task considers line, ring, and crossed-ring graphs, which induce increasing levels of propagation difficulty. We use the same data, hyperparameter search space, and experimental protocol as in \citep{gravina_swan}.

\textbf{Results.} \Cref{fig:graph_transfer} reports the mean test MSE and standard deviation over three seeds as a function of the distance between source and target nodes. \model{} achieves the lowest error at almost all evaluated distances across the three topologies, while most competing methods progressively degrade as the propagation distance increases. In particular, \model{} consistently improves over non-dissipative GNNs such as A-DGN, SWAN, and PH-DGN, while also outperforming GraphCON and SONAR, which share its oscillatory inductive bias but lack the richer matrix-valued transport provided by the sheaf structure. It also compares favorably with the graph Transformer GPS, despite relying only on local propagation. Overall, these results show that \model{} maintains effective information transfer as the distance between source and target nodes increases.

\begin{figure}[t]
    \centering
    \includegraphics[width=\linewidth]{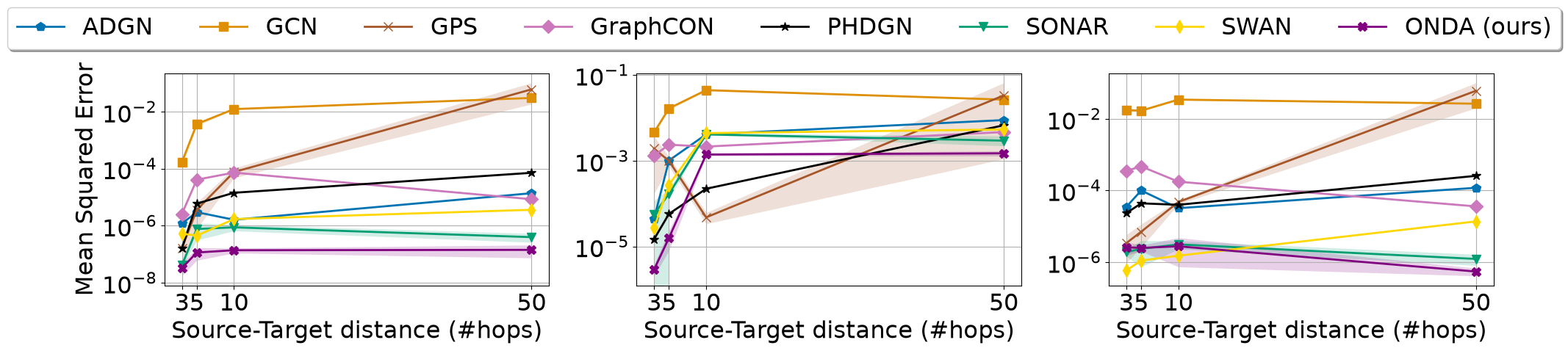}
     {\\\footnotesize \hspace{16mm}(a) Line \hspace{26mm}(b) Ring \hspace{25mm}(c) Crossed-Ring}
    \caption{Information transfer performance on (a) Line, (b) Ring, and (c) Crossed-Ring graphs. }
    \label{fig:graph_transfer}
\end{figure}
\section{Conclusion}\label{sec:conclusion}
We introduced \model{}, a sheaf neural network that couples learned matrix-valued transport with second-order oscillatory dynamics for effective long-range information propagation. Rather than relying on richer sheaf transport alone, \model{} uses the learned sheaf to determine how information is transformed between stalks, while wave-like dynamics allow these transformed signals to remain influential as they propagate across the graph. Our theoretical analysis shows that connected nodes can influence one another regardless of their distance, and our experiments demonstrate the effectiveness of \model{} across a large variety of tasks, including long-range propagation and short-range over-squashing, where our model sets a new bar for state-of-the-art performance, as well as heterophilic benchmarks.

\section{Acknowledgements}\label{sec:acknowledgements}
Jan-Willem Van Looy was supported by MSCA-DN CaLiForNIA, Project ID: 101119552. Rita Fioresi and Ferdinando Zanchetta acknowledge support by GNSAGA-Indam, INFN Gast Initiative,
PNRR MNESYS, PNRR National Center for HPC, Big Data and Quantum Computing CUP J33C22001170001, PNNR SIMQuSEC CUP J13C22000680006. This work was also supported by Horizon Europe EU projects MSCA-SE CaLIGOLA, Project ID: 101086123, COST Action CaLISTA CA21109.

\bibliography{references}
\bibliographystyle{iclr2027_conference}

\appendix
\section{Mathematical Background and Proofs}\label{section: appendix A}

This appendix provides the mathematical background and proofs for the theoretical results presented throughout Section \ref{sec:analysis}. First, Appendix~\ref{subsec_prel_app} gathers some additional notation and conventions on inner products. \Cref{subsec:app_sheaf_lap} establishes the spectral bound for the normalized sheaf Laplacian and includes a description of the directed sheaf operator. Appendix~\ref{subsec:app_sol_wave_eq} derives the closed-form solution of the continuous sheaf wave equation, while Appendices~\ref{subsec:proof_pathwise_source_target_response} and \ref{app:proof_prop2} prove the pathwise and spectral representations of the source-to-target response, respectively. Finally, Appendix~\ref{subsec:app_proof_th1} proves the discrete-time result from \Cref{prop:discrete_source_target_response}, and Corollary \ref{cor:disctocont} shows that the discrete response converges to the continuous response as the step size tends to zero.
Throughout, all vector spaces are over \(\mathbb{R}\).

\subsection{Preliminaries on inner products and norms} \label{subsec_prel_app} 
We recall basic definitions and facts concerning inner products and norms on finite-dimensional spaces. This material can be found in standard introductory textbooks to functional analysis, see for instance \citep{conway2019course}. For each vertex \( u \in \mathcal{V} \), denote by \( \langle \cdot, \cdot \rangle_{\gF(u)}\) a fixed real inner product on \( \gF(u)\) and write \( \| \mX \|_{\gF(u) }^2 = \langle \mX , \mX \rangle_{\gF(u)}\) for \(\mX \in \mathcal{F}(u)\). The cochain space \( \gH = C^0(\gG ; \gF) \) carries the direct sum inner product, defined by \( \langle \mX, \mY \rangle_{ \gH} := \sum_{u\in\gV} \langle\mX_u, \mY_u\rangle_{\gF(u)} \) and \( \|\mX\|_{\gH}^2 =\sum_{u\in\gV}\|\mX_u\|_{\gF(u)}^2.\) All matrix representations are taken with respect to fixed orthonormal bases of the vertex and edge stalks.

For a linear map \( A: F\to G \) between finite-dimensional inner-product spaces, we use the \emph{induced operator norm}, defined as \[ \|A\| := \sup_{ \| \mX \| _ F \leq 1 }\| A \mX \|_G.\] Since we consider finite-dimensional vector spaces, this is a finite quantity and satisfies \(\| A\mX \|_G \leq \|A\| \| \mX\|_F\). Additionally, for \( c \in \mathbb{R}\), \( A, B \in \operatorname{Hom}_\mathbb{R}(F,G)\), \( D \in \operatorname{Hom}_\mathbb{R}(E , F)\), the following properties hold:
\[ \| c A \| = |c| \| A \|,
\qquad \|A+B\|\leq\|A\| + \|B\|,  \qquad
\|AD \|\leq \|A\| \|D\|. \]
Recall that a sequence \( A_n: F \to G\) \emph{converges to \(A\) in operator norm} if \(\| A_n - A \| \to 0\). An operator series \( \sum_{n=0}^{\infty} A_n\) \emph{converges absolutely}  in operator norm if \( \sum_{n=0}^{\infty} \| A_n \| < \infty\). Since \(F\) and \(G\) are finite-dimensional, the space of linear maps
\( F\to G \) is complete in the operator norm. Thus, absolute convergence guarantees convergence of the partial sums in operator norm. 

For \(A,B\in\operatorname{Hom}_{\mathbb R}(F,G)\), we also use the \emph{Hilbert--Schmidt inner product} \[ \langle A,B \rangle_{\mathrm{HS}}  := \operatorname{tr}(A^*B),
 \qquad \|A\|_{\mathrm{HS}}^2 := \langle A,A\rangle_{\mathrm{HS}}. \]
 Since $\operatorname{Hom}_{\mathbb R}(F,G)$ is finite-dimensional, all norms on this space are equivalent. Thus, convergence of operator series is the same in the operator and Hilbert--Schmidt norms, as is absolute convergence.

\subsection{Sheaf Laplacian and Spectral bounds} \label{subsec:app_sheaf_lap}

Recall that we denote by \( \mathcal{D}_{ \gF}\) the block-diagonal part of \( L_{\gF} \), given by \( (\gD_{\gF})_{vv} = \sum_{e \ni v}\gF_{v,e}^{*} \gF_{v,e}. \) The symmetrically normalized sheaf Laplacian is given by \( \Delta_{ \gF} = \gD_{\gF}^{-1/2} L_{ \gF} \gD_{\gF}^{-1/2}. \) Here, \( \gD_{\gF}^{-1/2} \) denotes the Moore--Penrose pseudoinverse of \( \gD_{\gF}^{1/2} \). It acts by \( \mu^{-1/2} \) on each eigenspace of \( \gD_{\gF} \) with eigenvalue \( \mu>0 \), and by zero on \( \ker \gD_{\gF} \). We recall the following well-known fact, presented here with a self-contained elementary proof. A more general version of this result can be found in \citep{hansen2019toward}. 

\begin{lemma}\label{lem:normalized_sheaf_spectrum}
The symmetrically normalized sheaf Laplacian \( \Delta_{\gF} \) is self-adjoint and all of its eigenvalues lie in $[0,2]$.
\end{lemma}
\begin{proof}
Since \( L_{\gF} \) and \( \gD_{\gF}^{-1/2}\) are self-adjoint, also \( \Delta_{\gF}^{*} = \gD_{\gF}^{-1/2} L_{\gF}^{*} \gD_{\gF}^{-1/2} = \Delta_{\gF}.\)

Let \( (q_j)_j \) be an orthonormal eigenbasis of \( \gD_{\gF} \), with eigenvalues \( \mu_j \geq 0 \). For \(\mX = \sum_j\alpha_jq_j \), we have \[ \mY = \gD_{ \gF}^{-1/2}\mX = \sum_{\mu_j > 0} \frac{ \alpha_j}{ \sqrt{\mu_j} } q_j,  \qquad \langle \mY, \gD_{\gF}\mY\rangle = \sum_{\mu_j>0}|\alpha_j|^2 \leq \sum_j|\alpha_j|^2 = \|\mX\|^2.\]

Substituting \( \mY = \gD_{\gF}^{-1/2}\mX \), we obtain \[ \begin{aligned} 0 \leq \langle \mX, \Delta_{\gF} \mX \rangle & = \langle \mY, L_{\gF} \mY \rangle = \sum_{e=\{u,v\} \in \gE} \| \gF_{u,e} \mY_u - \gF_{v,e}\mY_v \|^2 \\ & \leq 2 \sum_{e=\{u,v\}\in\gE} \bigl( \|\gF_{u,e}\mY_u\|^2 + \|\gF_{v,e}\mY_v\|^2 \bigr) \\  & = 2 \sum_{v\in\gV}\sum_{e\ni v}
\langle\mY_v, \gF_{v,e}^{*}\gF_{v,e}\mY_v \rangle = 2 \langle \mY, \gD_{\gF}\mY\rangle \leq2\|\mX\|^2, \end{aligned} \]
where the inequality in the second line uses \( \|a - b\|^2 \leq 2\|a\|^2 + 2\|b\|^2 \), and the final inequality uses the bound on
\( \langle \mY, \gD_{\gF}\mY \rangle \) established above. Hence, for any eigenvalue \(\lambda\) with non-zero eigenvector \(\mX\), we find \( 0 \leq \lambda = \frac{ \langle \mX, \Delta_{\gF}\mX\rangle}{\|\mX\|^2} \leq 2. \)
\end{proof}

\paragraph{Directed variant.}\label{paragraph:dir_var} Here we give the precise description of the directed version of the sheaf Laplacian as mentioned in \Cref{par:discr_impl_detail}. For each edge \(e=\{u,v\}\), consider two oriented copies  \(e_u=(u\leftarrow v)\) and \(e_v=(v\leftarrow u) \). Each copy \( e_z \), \(z \in \{ u, v \}\), carries an edge stalk \(\gF(e_z)\) and independently learned restriction maps \( \gF_{w,e_z}: \gF(w) \to \gF(e_z), \ w\in \{u,v\}. \) The nodewise expression is given by \[ (L_{\gF}^{ \mathrm{dir}}\mX)_v = \sum_{e = \{u,v\}\in\gE} \gF_{v,e_v}^{*} \bigl( \gF_{v,e_v}\mX_v-\gF_{u,e_v} \mX_u \bigr). \] This operator is not guaranteed to be self-adjoint. Identifying \(\gF(e_u)=\gF(e_v) = \gF(e)\) and forcing \( \gF_{w,e_u} = \gF_{w,e_v} = \gF_{w,e}\) for \(w \in \{u,v\}\) recovers \(L_{\gF}\). 

\subsection{Solution of the sheaf wave equation} \label{subsec:app_sol_wave_eq}
We derive the solution of \Cref{eq:fixed_operator_wave} by rewriting the second-order dynamics as a first-order linear system and computing its operator exponential. 
\begin{lemma}
Let $L_{\gF}$ be fixed. The unique solution of \[ \ddot{\mX}(t) + L_\gF \mX(t) = 0, \] with initial conditions \(  \mX(0) = \overline{\mX},\  \dot{\mX}(0) = \overline{\mV}, \) 
is given by \( \mX(t) = \gC_{\gF}(t)\overline{\mX}  + \gS_{\gF}(t)\overline{\mV}. \)\end{lemma}
\begin{proof}
For a time-dependent node state \(\mX(t)\in C^0( \gG;\gF)\) with velocity \( \mV(t) =\dot{\mX}(t),\) we follow usual terminology of second-order dynamics \citep{teschl2012ordinary} and introduce the \emph{phase-space state} \[ \mZ(t) := \begin{pmatrix} \mX(t) \\ \dot{\mX}(t) \end{pmatrix} \in\gH\oplus\gH. \]
In this notation, our problem becomes equivalent to the first-order system
\begin{equation}
    \dot{\mZ}(t) = G_{\gF} \mZ(t),
    \qquad G_{\gF} := 
    \begin{pmatrix}
        0 & I_{\gH} \\
        -L_{\gF} & 0
    \end{pmatrix},
    \qquad \mZ(0) = \begin{pmatrix}\overline{\mX}\\\overline{\mV}\end{pmatrix}.
    \label{eq:app_phase_space_system}
\end{equation}
Since \( \gH\oplus\gH\) is finite-dimensional,
\(G_{\gF}\) is bounded. Its exponential is defined by the absolutely convergent power series \( e^{tG_{\gF}} :=  \sum_{ n=0 }^{ \infty } \frac{ t^n}{ n!} G_{\gF}^n \). Since \[ G_{\gF}^2 = -\begin{pmatrix} L_{\gF} & 0\\  0 & L_{\gF} \end{pmatrix},\] the even and odd powers in the above power series give
\begin{equation*}
\begin{aligned}
    e^{tG_{\gF}} & = \sum_{k=0}^{\infty}
      (-1)^k\frac{t^{2k}}{(2k)!} \begin{pmatrix}
          L_{\gF}^k & 0 \\ 0 & L_{\gF}^k
      \end{pmatrix}  +\sum_{k=0}^{\infty}
      (-1)^k\frac{t^{2k+1}}{(2k+1)!}
      \begin{pmatrix}
          0&L_{\gF}^k\\-L_{\gF}^{k+1}&0
      \end{pmatrix}\\
    &=\begin{pmatrix}
        \gC_{\gF}(t)&\gS_{\gF}(t)\\
        -L_{\gF}\gS_{\gF}(t)
        &\gC_{\gF}(t)
      \end{pmatrix}.
\end{aligned}
\end{equation*}

The power series defining \( \gC_{\gF} \)
and \( \gS_{\gF} \) have infinite radius of
convergence and can therefore be differentiated termwise.
Hence \[ \begin{aligned} \gC_{\gF}'(t) 
    & = \sum_{k = 1}^{ \infty }(-1)^k\frac{t^{2k-1}}{(2k-1)!} L_{\gF}^k = -L_{\gF}\sum_{j=0}^{\infty}
      (-1)^j \frac{t^{2j+1}}{(2j+1)!} L_{\gF}^j
     = -L_{\gF}\gS_{\gF}(t),\\
    \gS_{\gF}'(t)
    &=\sum_{k=0}^{\infty} (-1)^k\frac{t^{2k}}{(2k)!}L_{\gF}^k = \gC_{\gF}(t).
\end{aligned} \]
Therefore, for \(
\mX(t) := \gC_{\gF}(t)\overline{\mX}  + \gS_{\gF}(t)\overline{\mV}, \) we find that \[ \ddot{\mX}(t) = - L_{\gF} \gC_{\gF}(t) \overline{\mX} - L_{\gF} \gS_{ \gF} (t)\overline{\mV} = - L_{\gF}\mX(t). \]
Moreover, the above identities at \(t=0\) give that \( \mX(0)=\overline{\mX}, \ \dot{\mX}(0)=\overline{\mV}\). Hence \( \mX(t) \) solves our problem. Uniqueness follows by the standard existence-and-uniqueness theorem for finite-dimensional linear ordinary differential equations.
\end{proof}

\subsection{Proof of Proposition \ref{prop:pathwise_source_target_response} }\label{subsec:proof_pathwise_source_target_response}
We first express the blocks of \( L_{\gF}^k \) as sums of compositions along length-\( k \) walks.

Recall that, for \( k \geq 0 \) and vertices \(u, v \in \mathcal{V}\), we denoted by \( \gW_k(u,v) \) the set of length-\( k \) walks \( \omega=(v_0 = u, v_1, \ldots, v_k = v)\), where either \( v_j = v_{j-1}\) or \( \{v_{j-1}, v_j\} \in \gE\) at each step. To each such \(\omega \in \gW_k(u,v) \), we associated \[ (L_{\gF})_\omega  = (L_{\gF})_{v_k v_{k-1}} \ldots (L_{\gF})_{v_1 v_0}: \gF(u) \rightarrow \gF(v). \] Allowing steps with \(v_j = v_{j-1}\) accounts for the diagonal blocks of \(L_{\gF}\) in these compositions. For \( k = 0 \), we put \(\gW_0 (u , u) = \{ (u) \}\), with empty composition equal to \( I_{\gF(u)} \).

\begin{lemma}\label{lem:walk_expansion}
For every \(k \geq 0 \) and \( u, v \in \gV \), \[(L_{\gF}^k)_{vu} = \sum_{\omega\in\gW_k(u,v)}( L_{\gF})_\omega : \gF(u) \rightarrow \gF(v). \] 
\end{lemma}
\vspace{-1.5em}
\begin{proof}
We proceed by induction on \(k\).

For \(k=0\), the result is clear. Suppose the result holds for some \(k \geq 0\). Then \[
\begin{aligned} 
(L_{\gF}^{k+1})_{vu} & = \sum_{a\in\gV} (L_{\gF})_{va}(L_{\gF}^k )_{au}
= \sum_{a\in\gV} \sum_{\omega\in\gW_k(u,a)}  (L_{\gF})_{va}(L_{\gF})_\omega. 
\end{aligned} \] 
    
Since \( (L_{\gF})_{va}=0 \ \text{whenever }  a\neq v \text{ and } \{a,v\}\notin\gE\), only those \(a\) for which \(a=v\) or \(\{a,v\}\in\gE\) can contribute. For every such term, if \( \omega=(v_0=u,\ldots,v_k=a)\in\gW_k(u,a),\) appending \(v\) gives the walk \( \widetilde{\omega} = (v_0=u, \ldots, v_k = a, v_{k+1} = v) \in \gW_{k+1}(u,v), \) with associated composition \[ (L_{\gF})_{\widetilde{\omega}} = (L_{\gF})_{va}(L_{\gF})_\omega.\] Hence, \((L_{\gF}^{k+1})_{vu} =\sum_{\omega\in\gW_{k+1}(u,v)}
(L_{\gF})_\omega,\) completing the induction.
\end{proof}

For two different nodes $v,u$, we define their \emph{distance}, denoted by $\operatorname{dist}_{ \gG}( u, v)$, as the minimum $k \ge 0$ such that $ \gW_k(u,v) \neq \emptyset $ if they lie in the same connected component, and $ \operatorname{dist}_{ \gG}( u, v) = \infty $ otherwise. Also, $\operatorname{dist}_{ \gG}( u, u)=0$.

\emph{Proof of Proposition \ref{prop:pathwise_source_target_response} }. For \( r = \operatorname{dist}_{ \gG }(u,v) < \infty\), we have \(\gW_k(u,v) = \emptyset \) for \(k<r\), and hence \( (L_{\gF}^k)_{vu}=0\) for these values of \(k\). Taking the \((v,u)\)-block of the series defining \( \gC_{\gF}(t)\) gives \[ J^{\gF}_{v\leftarrow u}(t) =\sum_{ k = 0}^{ \infty}( -1 )^k\frac{t^{2k}}{(2k)!}(L_{\gF}^k)_{ v u } = \sum_{k = r}^{\infty} (-1)^k\frac{ t^{2k}}{( 2k)!} \sum_{\omega\in\gW_k(u,v)}(L_{\gF})_\omega. \]

At \( k = r \), these walks are precisely the shortest paths. Indeed, a stationary step or a repeated vertex would allow a segment to be removed and thereby produce a shorter walk from \( u \) to \( v \), contrary to the definition of \( r \).

For the last statement, if \( u \) and \( v \) lie in different connected components, then \( \mathcal{W}_k ( u, v) = \emptyset \) for every \( k \). Every block \((L_{\gF}^k)_{vu}\) therefore vanishes, and \(J^{\gF}_{v\leftarrow u}(t)=0\) for all \(t\). \qed 

\subsection{Proof of Proposition \ref{prop:spectral_stalk_sensitivity}}\label{app:proof_prop2}
This result applies to the positive semidefinite self-adjoint sheaf Laplacian and uses the orthogonal projectors onto the eigenspaces \( P_\lambda \) indexed by distinct eigenvalues given by the finite-dimensional spectral theorem \citep{Halmos1974}. These projectors satisfy \( L_{\gF} P_\lambda = \lambda P_\lambda,\) as well as the identities \( P_\lambda^* = P_\lambda, \ P_\lambda P_\mu = \delta_{\lambda \mu} P_\lambda , \) and \( \sum_{ \lambda \in \operatorname{spec}( L_{\gF}) } P_\lambda=I_{\gH}. \)

\emph{Proof of Proposition \ref{prop:spectral_stalk_sensitivity} }. Substituting the above identities into the power series gives \[ \begin{aligned} \gC_{ \gF } (t) P_\lambda & = \sum_{k=0}^{\infty} (-1)^k\frac{ t^{2k}}{ (2k)! } L_{\gF}^kP_\lambda \\ & = \left( \sum_{k=0}^{\infty} (-1)^k\frac{t^{2k}\lambda^k}{(2k)!} \right) P_\lambda = \cos( t \sqrt{\lambda}) P_\lambda. \end{aligned} \] Since the spectral projectors sum to \(I_{\gH}\), also
\[ \gC_{\gF}(t) = \gC_{\gF}(t) \sum_\lambda P_\lambda = P_0 + \sum_{\lambda>0} \cos(t\sqrt{\lambda})P_\lambda, \]
where \( P_0 \) is the projector onto \( \ker L_{\gF}\), with \(P_0 = 0\) when \( 0 \notin \operatorname{spec}(L_{\gF})\). Applying \(\pi_v\) on the left and \( \iota_u \) on the right proves the result. \qed 

\subsection{Proof of Theorem  \ref{prop:discrete_source_target_response} }\label{subsec:app_proof_th1}
For the proof of this result, we start with two preliminary lemmas. First, Lemma~\ref{lem:discrete_polynomial_response} expresses the discrete response as a polynomial in \( L_{\gF}\). Afterwards, we solve the corresponding scalar recurrence to obtain the spectral formula and boundedness in Lemma \ref{lem:scalar_discrete_wave_mode}. The spectral formula and boundedness in
Theorem~\ref{prop:discrete_source_target_response} then follow directly. To prove that a response which is non-zero at some step cannot converge to zero, we compute the Ces\`aro mean of its squared Hilbert--Schmidt norm. Finally, \Cref{cor:disctocont} shows that the discrete spectral coefficients converge to their continuous counterparts as \(h\to 0\) with \(\ell h=T\) fixed.

Let \((\mX^{(\ell)}, \mV^{(\ell)})\) denote the underlying trajectory of 
\Cref{eq:onda_discrete_updates} with fixed restriction maps, identity weights, and no dissipation or external forcing, \ie
\begin{equation}
 \mV^{(\ell+1)} = \mV^{(\ell)} - hL_{\gF}\mX^{(\ell)}, \qquad \mX^{(\ell + 1)} = \mX^{(\ell)} + h\mV^{(\ell+1)}, 
\label{eq:onda_reduced_discrete_updates}
\end{equation}

and let \( (\widetilde{\mX}^{(\ell)}, \widetilde{\mV}^{(\ell)})\) denote the trajectory obtained by replacing \( \overline{\mX} \) with \( \overline{\mX} + \iota_u \mW \), while keeping the initial velocity fixed. Write \[ \Delta \mX^{(\ell)} := \widetilde{\mX}^{(\ell)} - \mX^{(\ell)}, \qquad \Delta\mV^{(\ell)} := \widetilde{\mV}^{(\ell)} - \mV^{(\ell)}, \] so that \(\Delta \overline{\mX} = \iota_u \mW \) and \(\Delta \overline{\mV} = 0\).
The following lemma expresses the resulting response as a polynomial in \(L_{\gF}\).

\begin{lemma}\label{lem:discrete_polynomial_response}
Let \( p_{0,h}(z)=1\), \( p_{1,h} (z) = 1-h^2z\), and, for \(\ell\geq1\), \[ p_{\ell+1, h}(z) = (2 - h^2 z) p_{ \ell, h}(z) - p_{\ell-1,h}(z) .\] Then, for every integer \( \ell \geq 0 \), \( \Delta \mX^{(\ell)} = p_{\ell,h} (L_{\gF}) \iota_u\mW, \) and \( J^{\gF}_{v\leftarrow u}[\ell,h] = \bigl(p_{\ell,h}(L_{\gF})\bigr)_{vu},\) where \(\deg p_{\ell,h}\leq \ell\). In particular, if 
\( \operatorname{dist}_{\gG} (u,v) > \ell \), then \( J^{\gF}_{v \leftarrow u}[\ell,h]=0\).
\end{lemma}

\begin{proof}
Taking differences in the update \Cref{eq:onda_reduced_discrete_updates} gives \[ \Delta \mV^{(\ell+1)} = \Delta \mV^{(\ell)} - h L_{\gF}\Delta \mX^{(\ell)}, \qquad \Delta \mX^{(\ell+1)} = \Delta\mX^{(\ell)}+h\Delta\mV^{(\ell + 1)}.\]
For \( \ell = 0 \), this becomes \( \Delta \mX^{(1)} = (I_{\gH} -h^2 L_{\gF})\Delta\mX^{(0)} = p_{1,h}(L_{\gF})\iota_u\mW.\) 

For \(\ell \geq 1\), the above update gives \( h \Delta \mV^{(\ell)} = \Delta \mX^{(\ell)} - \Delta\mX^{(\ell-1)} \). We substitute this into the next one, giving
\[  \begin{aligned}
\Delta \mX^{(\ell+1)} & = \Delta \mX^{(\ell)} + h \Delta\mV^{(\ell)} - h^2 L_{\gF} \Delta \mX^{(\ell)} = ( 2 I_{\gH} - h^2 L_{\gF}) \Delta \mX^{(\ell)} - \Delta \mX^{(\ell-1)}.
\end{aligned} \]
Together with the initial values, this recurrence implies by induction that \( \Delta \mX^{(\ell)} = p_{\ell,h}(L_{\gF})\iota_u\mW.\) After applying \( \pi_v \), and using that the identity holds for every \(\mW \in \gF(u)\), this gives \( J^{\gF}_{v \leftarrow u}[\ell , h] = \pi_v p_{\ell , h} (L_{\gF}) \iota_u = \bigl(p_{\ell, h} (L_{\gF})\bigr)_{vu}.\)

The recurrence defining \( p_{\ell,h} \) directly implies that \( \deg p_{\ell,h} \leq \ell \). 

If \( \operatorname{dist}_{\gG}(u,v)>\ell\), then \( \mathcal{W}_j(u,v) = \emptyset \) for every \(j\leq \ell\). By Lemma \ref{lem:walk_expansion}, \( (L_{\gF}^j)_{vu}=0\) for all such \( j \), and hence \( J^{\gF}_{v\leftarrow u}[\ell,h] = \bigl(p_{\ell,h}(L_{\gF})\bigr)_{vu} = 0. \)
\end{proof}

We next solve the scalar recurrence defining \( p_{\ell,h} \). For this and the subsequent proof, we use the following trigonometric identities, valid for \( a, b, \theta \in \mathbb{R} \):
\begin{equation}
    \cos( a + b) + \cos(a - b) =  2\cos(a) \cos(b), \label{eq:cos_sum_difference}
\end{equation}
\begin{equation}
    \sin(a + b) - \sin(a - b) = 2 \cos( a ) \sin( b ),
    \label{eq:sin_sum_difference}
\end{equation}
\begin{equation}
\cos \theta = 1 - 2 \sin^2 \left ( \frac{\theta}{2}\right), \qquad \cos^2 \theta = \frac{1+\cos(2\theta)}{2}.
\label{eq:cos_squared}
\end{equation}

\begin{lemma}\label{lem:scalar_discrete_wave_mode}
Let \( h > 0 \) and \( \lambda \geq 0\) satisfy \( h^2 \lambda < 4 \), and define \( \theta_h(\lambda) := 2 \arcsin \left( \frac{h\sqrt{\lambda}}{2} \right). \) Then, for every integer \(\ell \geq 0 \), \[ p_{\ell,h}(\lambda)   = \frac{ \cos \left( (\ell + \tfrac12) \theta_h( \lambda)\right)}{ \cos \left( \theta_h( \lambda)/ 2 \right) }.\]

Moreover, \( \left|p_{\ell ,h}( \lambda ) \right| \leq \frac{1 }{ \sqrt{1-h^2\lambda/4}} .\)
\end{lemma}

\begin{proof} Since \( 0 \leq h \sqrt{ \lambda}/ 2 < 1 \), we have \( \theta_h( \lambda ) \in [ 0, \pi ) \). By definition of \( \theta_h( \lambda ) \), also \( \sin \left( \frac{ \theta_h (\lambda)}{2} \right) = \frac{h\sqrt{\lambda}}{2}.\) Using \Cref{eq:cos_squared}, we find 
\begin{equation}\label{eq:interm_eq_lemma}
2 \cos \theta_h (\lambda) = 2 \left( 1 - 2 \sin^2 \left(\frac{\theta_h( \lambda )}{2} \right)\right) = 2 - h^2 \lambda. \end{equation} 

Define \[ q_\ell := \frac{ \cos(( \ell + \tfrac12) \theta_h( \lambda ))} {\cos ( \theta_h(\lambda)/ 2)}. \] 
For \( \ell = 0 \), \( q_0 = 1 = p_{0,h} (\lambda )\). 

Applying \Cref{eq:cos_sum_difference} to \( a =  \theta_h(\lambda), \ b = \theta_h(\lambda)/ 2\) and dividing by \( \cos( \theta_h( \lambda) / 2) \) gives \[ \frac{ \cos( 3 \theta_h ( \lambda ) / 2 )}{ \cos( \theta_h( \lambda) / 2)} = 2 \cos \theta_h(\lambda) - 1, \] so that \( q_1 = 2 \cos \theta_h(\lambda) - 1 =  1 - h^2 \lambda = p_{1,h}(\lambda) \), where we used \Cref{eq:interm_eq_lemma} and the definition of \(p_{1,h}(\lambda)\).

For \( \ell \geq 1 \), again applying \Cref{eq:cos_sum_difference} with \(a = (\ell+\tfrac12) \theta_h(\lambda) \) and \( b = \theta_h(\lambda)\) gives \[ 
\cos \left( (\ell + \tfrac32 ) \theta_h(\lambda)  \right ) + \cos \left( (\ell - \tfrac12) \theta_h(\lambda) \right) = 2 \cos \theta_h(\lambda) \cos \left( (\ell + \tfrac12)\theta_h(\lambda)\right).
\]

Dividing this expression by \(\cos ( \theta_h( \lambda) / 2)\), and using the definition of \( q_\ell \), we find \[ q_{ \ell + 1 } + q_{\ell - 1} = 2 \cos \theta_h( \lambda )\, q_\ell. \] By \Cref{eq:interm_eq_lemma}, this can be rearranged to \( q_{ \ell + 1} = 2 \cos \theta_h ( \lambda )  q_\ell - q_{\ell - 1} = ( 2 - h^2 \lambda) q_\ell - q_{\ell - 1}.\)

In particular, \( q_\ell \) has the same initial values and recurrence as \( p_{\ell,h} ( \lambda ) \), and hence \[ p_{\ell,h} (\lambda) = \frac{ \cos \left( (\ell + \tfrac12) \theta_h( \lambda) \right) }{ \cos \left( \theta_h( \lambda) /2 \right) }. \]

For the final statement, we use \Cref{eq:cos_squared,eq:interm_eq_lemma} to find \begin{equation}\label{eq_cos_squared_extra}
\cos^2( \theta_h ( \lambda ) /2 ) = \frac{ 1 + \cos \theta_h ( \lambda)}{2} 
= 1 - \frac{ h^2 \lambda}{4}.
\end{equation}
Hence, \( \cos \left( \frac{ \theta_h( \lambda)}{ 2}\right) = \sqrt{ 1 - \frac{ h^2\lambda }{ 4}} > 0 \), and \( \left| p_{ \ell , h}( \lambda )\right| \leq \frac{1}{ \cos \left ( \theta_h (\lambda)/2 \right) } = 
\frac{1}{\sqrt{1 - h^2 \lambda / 4}} \).
\end{proof}

\emph{Proof of Theorem \ref{prop:discrete_source_target_response} }. 
By Lemma \ref{lem:discrete_polynomial_response} and the spectral decomposition of \( L_{\gF} \),
\[ J^{\gF}_{v \leftarrow u}[ \ell, h]  = \bigl( p_{\ell,h}( L_{\gF} ) \bigr )_{vu} = \sum_\lambda p_{\ell,h}(\lambda)(P_\lambda)_{vu}. \] Since \( h^2 \lambda_{\max} (L_{\gF}) < 4\), Lemma \ref{lem:scalar_discrete_wave_mode} gives
\Cref{eq:discrete_wave_spectral_sensitivity}.

The fact that \(  J^{\gF}_{v\leftarrow u}[\ell,h]=0 \ \text{whenever } \operatorname{ dist }_{ \gG}( u, v) > \ell \) follows again from Lemma \ref{lem:discrete_polynomial_response}.

The bound in Lemma \ref{lem:scalar_discrete_wave_mode} gives \[ \left\|J^{\gF}_{v\leftarrow u}[\ell,h]\right\|   \leq \sum_\lambda \frac{\| ( P_\lambda )_{ v u} \|} {\sqrt{1 - h^2 \lambda / 4 }}. \]
The right-hand side is finite under the step-size assumption and independent of \( \ell \), showing that the response is bounded as the propagation depth varies.

For the final statement, fix \( u , v \in \mathcal{V} \) and write \( J_\ell := J^{ \gF }_{ v \leftarrow u }[ \ell , h ] \). To exclude that $J_\ell \longrightarrow 0$, it suffices to show that the Ces\`aro average \( \frac 1 N \sum_{ \ell = 0 }^{ N - 1} \| J_\ell \|_{\mathrm{HS}}^2\) has a strictly positive limit whenever the response is non-zero at some step.

We use the Hilbert–Schmidt norm as introduced in \Cref{subsec_prel_app} to expand the inner products
\[ \begin{aligned}
    \|J_\ell\|_{\mathrm{HS}}^2
    &= \|(P_0)_{vu}\|_{\mathrm{HS}}^2
    +2\sum_{\lambda>0} p_{\ell,h}(\lambda) \left \langle(P_0)_{vu}, (P_\lambda)_{vu}  \right\rangle_{\mathrm{HS}}\\
    & \quad + \sum_{\lambda,\mu>0}  p_{\ell,h}(\lambda)p_{\ell,h}(\mu) \left\langle( P_\lambda )_{vu} , (P_\mu)_{vu}    \right\rangle_{\mathrm{HS}}\\
    &= \|(P_0)_{vu}\|_{\mathrm{HS}}^2
    + 2\sum_{\lambda>0}p_{\ell,h}(\lambda)
      \left\langle(P_0)_{vu},(P_\lambda)_{vu}
      \right\rangle_{\mathrm{HS}} + \sum_{\lambda>0} p_{\ell,h}(\lambda)^2 \|( P_\lambda )_{vu}\|_{ \mathrm{HS}}^2 \\
    & \quad + 2 \sum_{0 < \lambda < \mu} p_{\ell, h }( \lambda) p_{\ell,h}( \mu)
      \left \langle( P_\lambda )_{vu} , ( P_\mu )_{vu}  \right\rangle_{\mathrm{HS}}.
\end{aligned} \]
The last equality separates the diagonal terms \( \lambda = \mu\) and uses symmetry of the real Hilbert--Schmidt inner product.

We analyze the three oscillatory sums separately, using a shared averaging identity. For fixed \(\omega \notin2 \pi\mathbb Z\), applying \Cref{eq:sin_sum_difference} with \( a = (\ell + \tfrac12 ) \omega \) and \( b = \omega/ 2 \) gives \[ \sin( ( \ell + 1 )\omega) - \sin( \ell \omega ) = 2\sin( \omega/2) \cos( (\ell + \tfrac12)\omega). \]
Summing over \(\ell=0, \ldots, N-1\), the left-hand side telescopes.
Consequently,
\begin{equation}\label{eq:discrete_cosine_average}
    T_N( \omega ) := \frac1N \sum_{\ell = 0}^{N - 1} \cos( (\ell + \tfrac12) \omega) = \frac{ \sin(N \omega)}{2N \sin(\omega/2)}     \longrightarrow 0 .
\end{equation}
Here \(\sin(\omega/2)\neq0\) is fixed, while
\(|\sin(N\omega)|\leq1\).

\emph{Term 1}.
For \(\lambda > 0\), we have \( \theta_h( \lambda ) \in ( 0, \pi) \). Lemma \ref{lem:scalar_discrete_wave_mode} and
\Cref{eq:discrete_cosine_average} therefore give \[ \frac1N\sum_{\ell=0}^{N-1}p_{\ell,h} (\lambda) = \frac{T_N(\theta_h(\lambda))}{\cos(\theta_h(\lambda)/2)}  \longrightarrow 0. \]

\emph{Term 2.} For \( \lambda > 0 \), using \Cref{eq:cos_squared}, we obtain
\[\begin{aligned}
    \frac1N \sum_{ \ell = 0 }^{ N-1 } p_{\ell,h}(\lambda)^2
    & = \frac{1} {N \cos^2 ( \theta_h( \lambda ) / 2)} \sum_{\ell=0}^{N-1} \cos^2((\ell + \tfrac12)\theta_h( \lambda)) \\
    &=  \frac{1 + T_N( 2 \theta_h( \lambda))}{ 2 \cos^2( \theta_h(\lambda) /2)}
     \longrightarrow \frac{1}{2\cos^2(\theta_h(\lambda)/2)}.
\end{aligned} \]
Here $T_N(2\theta_h(\lambda))\to0$ because of \Cref{eq:discrete_cosine_average} as \(2\theta_h(\lambda)\in(0,2\pi)\).

\emph{Term 3.} For distinct positive eigenvalues \(\lambda ,\mu\), applying \Cref{eq:cos_sum_difference} gives
\[ \frac1N\sum_{\ell = 0}^{N-1}  p_{\ell,h}(\lambda)p_{\ell,h}(\mu) = \frac{T_N(\theta_h(\lambda)-\theta_h( \mu)) + T_N( \theta_h(\lambda) + \theta_h(\mu))}{2 \cos(\theta_h ( \lambda ) /2 ) \cos(\theta_h(\mu)/2)} \  \longrightarrow0. \]

Indeed, as \(\theta_h\) is strictly increasing,
\[
    \theta_h(\lambda)-\theta_h(\mu)
    \notin 2\pi\mathbb Z,
    \qquad
    \theta_h(\lambda)+\theta_h(\mu)\in(0,2\pi).
\]
So
\Cref{eq:discrete_cosine_average} applies to both cosine terms.

Since the spectrum is finite, combining these limits gives 
\[ \lim_{N\to\infty} \frac1N \sum_{\ell = 0}^{N-1}
    \| J_\ell \|_{\mathrm{HS}}^2 = \|( P_0 )_{vu}\|_{\mathrm{HS}}^2 + \sum_{\lambda>0}  \frac{\| (P_\lambda )_{vu}\|_{\mathrm{HS}}^2} {2\cos^2(\theta_h(\lambda)/2)}. \]
If \(J_\ell\neq0\) at some step, its spectral representation implies that at least one block \((P_\lambda)_{vu}\) is non-zero.
The right-hand side is therefore strictly positive. 

If \(J_\ell \) converged to zero in operator norm, the equivalence of norms in Appendix \ref{subsec_prel_app} would imply \( \| J_\ell \|_{\mathrm{HS}}^2 \to 0 \). Its averages would then also converge to zero, contradicting the positive limit above. Thus, a response that is non-zero at some step cannot converge to zero as \(\ell \to \infty\).
\qed

To conclude, we show that the discrete spectral response converges to its continuous counterpart when the step size tends to zero at fixed propagation time.

\begin{corollary}\label{cor:disctocont}
Let \(L_{\gF}\) be fixed, self-adjoint, and positive semidefinite. Fix \(T>0\) and set \(h=T/\ell\). Then, for every eigenvalue \( \lambda \) of \( L_{\gF} \), \( \lim_{ \ell\to\infty }
p_{\ell,h}( \lambda) = \cos( T \sqrt{\lambda}).\) Consequently, for every pair of nodes \( u, v\),
\[ \lim_{\ell \to \infty} \left\| J^{\gF}_{v \leftarrow u}[\ell, T/\ell] - J^{ \gF }_{ v\leftarrow u} (T) \right \|=0. \]
\end{corollary}
\begin{proof}
Since \(h=T/\ell \to 0\), the stability condition \( h^2 \lambda_{\max} (L_{\gF}) < 4\) holds for all sufficiently large \( \ell \). We may therefore use the spectral coefficient from Lemma \ref{lem:scalar_discrete_wave_mode}.

For \( \lambda = 0 \), the coefficient equals \(1\), so the limit statement holds.

For \(\lambda > 0\), we have  \( \frac{\theta_h ( \lambda )}{h} = \sqrt{ \lambda}
\frac{\arcsin (h \sqrt{\lambda}/2)} {h\sqrt{\lambda}/2} \longrightarrow\sqrt{\lambda},\) where the last step uses the identity  \( \lim_{x \to 0} \frac{\arcsin(x)} {x} = 1\). Since \( \ell h = T \), it follows that
\( (\ell + \tfrac12) \theta_h( \lambda) = (T + \tfrac h 2 ) \frac{ \theta_h( \lambda)}{h} \longrightarrow T \sqrt{\lambda}.
\) Using \Cref{eq_cos_squared_extra}, also \( \cos \bigl ( \theta_h ( \lambda) /2 \bigr) = \sqrt{1 - h^2 \lambda/4 } \longrightarrow 1. \) Continuity of cosine therefore gives \[ p_{\ell,h}(\lambda) = \frac{\cos\bigl( ( \ell + \tfrac12)\theta_h(\lambda)\bigr)}{\cos\bigl(\theta_h(\lambda)/2\bigr)}
\longrightarrow \cos( T \sqrt{ \lambda } ). \] For the final statement, we have that 
\( \left\| J^{ \gF}_{ v \leftarrow u}[\ell, h ] - J^{ \gF}_{v \leftarrow u}( T ) \right \| \leq \sum_{ \lambda } \left| p_{\ell, h}( \lambda ) - \cos( T \sqrt{ \lambda })\right| \left\| ( P_\lambda )_{ vu }\right\|. \) For \( \ell \longrightarrow \infty \), this last term converges to zero by the above.
\end{proof}

\section{Additional Experimental Details}\label{sec:experimental-details} This section provides \Cref{sec:experiments} with details on restriction-map parametrizations, training protocols, and hyperparameter selection.

\subsection{Restriction-Map Parametrizations}
\label{app:restriction_parameterizations}
We consider four restriction-map parametrizations for \model{}, following \citep{bodnar2022neural, borgi2026sheafmpnn}. The maps \(\gF_{v,e}: \gF(v) \to \gF(e)\) act between \(d\)-dimensional node and edge stalks and take the following
forms: 
\begin{itemize}
    \item \emph{Diagonal:} The restriction map is constrained to
    $\mathcal{F}_{v,e}=\operatorname{diag}(\boldsymbol{\alpha}_{ v, e})$. Each stalk coordinate is therefore scaled independently, without mixing information across coordinates. This requires \( d \) scalar coefficients per restriction map.

    \item \emph{Orthogonal.}
    The restriction map satisfies $\mathcal{F}_{ v, e}^{ \top} \mathcal{ F}_{ v,e} = I_d$. These maps allow for coordinate mixing while preserving Euclidean norms and inner products. 
    
    \item \emph{Low-Rank.} The restriction map is represented in factorized form, $ \mathcal{F}_{v,e} = U_{v,e} V_{v,e}^{\top}$, with $ U_{v,e}, V_{v,e}\in\mathbb{R}^{d\times r}$ and $ 1 \leq r < d $. This permits coordinate mixing through an intermediate space of dimension \(r\), with \( \operatorname{rank}(\gF_{v,e})\leq r\).
    
    \item \emph{General.} The restriction map is an unconstrained matrix $ \mathcal{F}_{v,e} \in \mathbb{R}^{d\times d}$. This provides the most flexible local transport, allowing arbitrary linear transformations between stalks, at the cost of $\mathcal{O}(d^2)$ storage and application complexity.
\end{itemize}
We stress that these parametrizations only affect the restriction maps used to construct \(L_{\gF}\), while retaining the same \model{} update rule and oscillatory dynamics.

\subsection{Employed Baselines}\label{app:baselines_details}
In our experiments, the performance of our method is compared with various state-of-the-art GNN and SNN baselines from the literature. Specifically, we consider:
\begin{itemize}
    \item Classical MPNN-based methods, \ie ChebNet~\citep{chebnet}, GCN~\citep{gcn}, GraphSAGE~\citep{SAGE}, GAT~\citep{GAT}, GatedGCN~\citep{gatedgcn}, GIN~\citep{GIN}, 
    GINE~\citep{gine}, GCNII~\citep{gcnii}, 
    and CoGNN~\citep{finkelshtein2024cooperative};
    
    \item Heterophily-specific models, \ie H2GCN \citep{h2gcn}, CPGNN \citep{CPGNN}, FAGCN \citep{FAGCN}, GPR-GNN \citep{GPR_GNN},
FSGNN \citep{FSGNN}, GloGNN \citep{GloGNN}, GBK-GNN \citep{GBK_GNN}, and JacobiConv \citep{jacobiconv};
    
    \item DE-GNNs, \ie 
    GraphCON~\citep{rusch2022graph}, A-DGN~\citep{gravina_adgn}, SWAN~\citep{gravina_swan}, PH-DGN~\citep{gravina_phdgn}, SONAR~\citep{gravina_sonar};
    
    \item Graph Transformers, \ie Transformer~\citep{Vaswani2017, dwivedi2021generalization}, GT~\citep{ijcai2021p0214}, 
    GPS~\citep{graphgps}, GOAT~\citep{goat}, Exphormer~\citep{Shirzad2023}, NAGphormer \cite{chen2023nagphormer}, {GRIT}\citep{grit}
    ;
    
    \item Higher-Order GNNs, \ie 
    DRew~\citep{drew}
    ;
    
    \item Sheaf Neural Networks, \ie NSD~\citep{bodnar2022neural}, BuNN~\citep{bamberger2025bundle}, and CSNN~\citep{ribeiro2026cooperative}.
    

\end{itemize} 

\subsection{Hyperparameter Search Spaces for the Experiments in \Cref{sec:experiments}}\label{app:hyperparameters}
The hyperparameter search spaces used for \model{} are reported in \Cref{tab:echo_hyperparameter_grid} for the ECHO benchmark, \Cref{tab:barbell_hyperparameter_grid} for the Barbell benchmark, and \Cref{tab:transfer_hyperparameter_grid} for the Graph Transfer benchmark. To select the best hyperparameters, we perform a Bayesian optimization in the chosen hyperparameter space.

\begin{table}[h]
\centering
\caption{\textbf{Hyperparameter search space for the ECHO Benchmark.}}
\label{tab:echo_hyperparameter_grid}
\begin{tabular}{ll}
\toprule
Hyperparameter & Candidate values \\
\midrule
Hidden dimension \(H\)
    & $[36, 252]$ \\
Stalk dimension \(d\)
    & \(\{2,3,4\}\) \\
\model{} Blocks \(B\)
    & \(\{1,2,3,4\}\) \\
Inner iterations \(\mathcal{L}\)
    & $[1, 40]$  \\
Step size \(h\)
    & $[10^{-4}, 0.5]$ \\
Learning rate
    & $[10^{-5}, 10^{-2}]$\\
Operator & \(\{\text{classical}, \text{normalized}, \text{directional}\}\) \\
Sheaf type & \(\{\text{general}, \text{diagonal}, \text{orthogonal}\}\) \\
Dissipation
    & \(\{\text{False},\text{True}\}\) \\
Forcing
    & \(\{\text{False},\text{True}\}\) \\
Fixed Restriction maps
    & \(\{\text{False},\text{True}\}\) \\
\bottomrule
\end{tabular}
\end{table}
\begin{table}[h]
\centering
\caption{\textbf{Hyperparameter search space for the Barbell benchmark.}}
\label{tab:barbell_hyperparameter_grid}
\begin{tabular}{ll}
\toprule
Hyperparameter & Candidate values \\
\midrule
Hidden dimension \(H\)
    & \(\{128, 192\}\) \\
Stalk dimension \(d\)
    & \(\{2,6\}\) \\
\model{} Blocks \(B\)
    & \(\{1,2,4\}\) \\
Inner iterations \(\mathcal{L}\)
    & \(\{8, 10, 12, 14, 16, 20\}\) \\
Step size \(h\)
    & \(\{0.1, 0.2, 0.25, 0.33, 0.5, 0.6, 0.66, 0.75, 1.0\}\) \\
Learning rate
    & \(\{3\times10^{-4}\}\) \\
Operator & \(\{\text{classical}\}\) \\
Sheaf type & \(\{\text{general}, \text{diagonal}, \text{orthogonal}, \text{low-rank}\}\) \\
Dissipation
    & \(\{\text{False}\}\) \\
Forcing
    & \(\{\text{False}\}\) \\
Fixed Restriction maps
    & \(\{\text{True}\}\) \\
\bottomrule
\end{tabular}
\end{table}

\begin{table}[h]
\centering
\caption{Hyperparameter search space for the graph transfer task.}
\label{tab:transfer_hyperparameter_grid}
\begin{tabular}{ll}
\toprule
\textbf{Hyperparameter} & \textbf{Candidate values} \\
\midrule
Hidden dimension \(H\)
    & \(64\) \citep{gravina_swan} \\
Stalk dimension \(d\)
    & \(\{1,2,4\}\) \\
\model{} Blocks \(B\)
    & \(\{1,2\}\) \\
Inner iterations \(\mathcal{L}\)
    & distance x \(\{1, 2\}\) / $B$  \\
Step size \(h\)
    & \(\{0.05, 0.1, 0.5, 1.0\}\) \\
Learning rate
    & $10^{-3}$ \citep{gravina_swan} \\
Operator & \(\{\text{classical}, \text{normalized},\text{directional}\}\) \\
Sheaf type & \(\{\text{general}, \text{diagonal}\}\) \\
Dissipation
    & \(\{\text{False},\text{True}\}\) \\
Forcing
    & \(\{\text{False},\text{True}\}\) \\
Fixed Restriction maps
    & \(\{\text{False},\text{True}\}\) \\
\bottomrule
\end{tabular}
\end{table}

\section{Additional Results}\label{app:additional_results}
\subsection{Restriction Map Parametrization Ablation} \label{app:barbell_full_results}
In this section, we investigate how the parametrization of the restriction maps affects information propagation on the Barbell benchmark from \Cref{sec:barbell}. In \Cref{tab:barbell-results-full}, we compare four parametrizations, \ie diagonal, orthogonal, low-rank, and general maps, which differ in the expressivity of the learned transport between stalks. To disentangle the effect of the restriction-map parametrization from that of the oscillatory dynamics, we evaluate the same variants for both \model{} and NSD. This allows us to assess how the choice of sheaf transport affects information transfer across the severe topological bottleneck of the Barbell graphs. To focus on the propagation mechanism, we consider \model{} configurations with a single-layer MLP-based restriction-map learner, without forcing or dissipation.

\paragraph{Results.} 
Across all restriction-map parametrizations and graph sizes, \model{} substantially improves over the corresponding NSD variant, showing that its effectiveness is robust to the particular choice of sheaf transport. All four \model{} parametrizations solve the task at \(N=10\), with the diagonal and orthogonal variants remaining below the solving threshold across all tested sizes. Among them, diagonal maps provide the most consistent behavior, achieving the lowest mean error at both \(N=20\) and \(N=50\). The remaining parametrizations exhibit a stronger dependence on graph size. Notably, the most expressive general maps are not systematically preferable to more constrained parametrizations. Overall, these results indicate that richer restriction-map expressivity alone is not sufficient for effective information transfer across severe bottlenecks, while the oscillatory propagation of \model{} consistently improves upon diffusion-based NSD under matched restriction-map parametrizations.

\begin{table}[h]
\centering
\caption{Barbell additional results. Test MSE ($\times 10^{-3}$; $\downarrow$) on the Barbell task for graph sizes $N\in\{10,20,50\}$. All models report the results over the best configuration selected by validation, as mean$_{\pm\text{std}}$ over 5 seeds. For each column, we report the \one{best}, \two{second-best}, and \three{third-best} results.}
\label{tab:barbell-results}
\small
\vspace{1mm}
\setlength{\tabcolsep}{6pt}
\begin{tabular}{lccc}
\toprule
\textbf{Model}
& $\boldsymbol{N=10}$ $\downarrow$
& $\boldsymbol{N=20}$ $\downarrow$
& $\boldsymbol{N=50}$ $\downarrow$ \\
\midrule

GCN
& $884.200_{\pm 6.900}$
& $961.100_{\pm 12.300}$
& $860.000_{\pm 35.000}$ \\

GAT
& $7.900_{\pm 2.100}$
& $487.000_{\pm 342.200}$
& $852.000_{\pm 29.000}$ \\

GraphSAGE
& $455.500_{\pm 199.600}$
& $998.200_{\pm 6.500}$
& $876.000_{\pm 32.000}$ \\

GIN
& $900.500_{\pm 47.900}$
& $970.600_{\pm 3.900}$
& $985.000_{\pm 123.000}$ \\

GraphTransformer
& $670.300_{\pm 306.800}$
& $1041.200_{\pm 46.900}$
& $831.000_{\pm 13.000}$ \\

ChebNet
& $10.600_{\pm 1.200}$
& $983.600_{\pm 5.100}$
& $922.000_{\pm 47.000}$ \\

SONAR
& $462.800_{\pm 250.300}$
& $717.500_{\pm 249.600}$
& $809.000_{\pm 68.000}$ \\

\cmidrule(r){1-1}\cmidrule(l){2-4}

Diag-NSD
& $937.200_{\pm 11.400}$
& $1006.600_{\pm 22.000}$
& $858.000_{\pm 43.000}$ \\

Ortho-NSD
& $965.300_{\pm 23.000}$
& $1005.100_{\pm 7.800}$
& $891.000_{\pm 38.000}$ \\

LowRank-NSD
& $955.100_{\pm 6.300}$
& $1008.100_{\pm 16.400}$
& $825.000_{\pm 32.000}$ \\

General-NSD
& $938.100_{\pm 12.100}$
& $1034.200_{\pm 23.500}$
& $860.000_{\pm 31.000}$ \\

CSNN
& $473.900_{\pm 106.200}$
& $890.800_{\pm 265.400}$
& $742.000_{\pm 168.000}$ \\

\cmidrule(r){1-1}\cmidrule(l){2-4}

Diag-\model{}
& \two{$1.200_{\pm 0.400}$}
& \one{$0.900_{\pm 0.500}$}
& \one{$83.000_{\pm 85.000}$} \\

Ortho-\model{}
& \one{$0.700_{\pm 0.300}$}
& \three{$68.900_{\pm 165.900}$}
& \two{$92.000_{\pm 127.000}$} \\

LowRank-\model{}
& $22.800_{\pm 19.500}$
& $840.100_{\pm 307.900}$
& \three{$128.000_{\pm 154.000}$} \\

General-\model{}
& \three{$1.300_{\pm 0.300}$}
& \two{$2.200_{\pm 1.200}$}
& $616.000_{\pm 121.000}$ \\

\bottomrule
\end{tabular}
\end{table}

\subsection{Heterophilic Tasks}\label{sec:heterophilic}
Although \model{} is primarily designed to improve long-range information
propagation, we also evaluate its effectiveness on heterophilic tasks. 
We consider the heterophilic benchmarks from \cite{platonov2023a}, and compare \model{} against conventional MPNNs, graph Transformers, heterophily-specific architectures, and recent sheaf neural networks (see also \Cref{app:baselines_details}). 

\paragraph{Hyperparameter Space.}\label{app:hyperparams}
Table \ref{tab:heterophily_hyperparameter_grid} reports the hyperparameter values considered for the heterophilic benchmarks. Hyperparameters were selected through sequential local searches over subsets of these candidate values as well as Bayesian optimization. Final configurations were selected according to validation performance. Similarly to \cite{bamberger2025bundle,ribeiro2026cooperative}, for these tasks we optionally learn restriction maps from node representations produced by a GNN encoder. Specifically, in Table \ref{tab:heterophily_hyperparameter_grid}, \emph{GNN depth}
denotes the number of message-passing layers, and \emph{GNN hidden
dimension} denotes the encoder's hidden width. At depth zero,
an MLP learns the restriction maps directly from the endpoint
node features, similar to \citep{bodnar2022neural}. 

\begin{table}[h]
\centering
\caption{Hyperparameter search space for the heterophilic benchmarks.}
\label{tab:heterophily_hyperparameter_grid}
\begin{tabular}{ll}
\toprule
\textbf{Hyperparameter} & \textbf{Candidate values} \\
\midrule
Hidden dimension \(H\)
    & \(\{64,128,256,512\}\) \\
Stalk dimension \(d\)
    & \(\{2,4,8\}\) \\
\model{} Blocks \(B\)
    & \(\{2,3,4,6,8,10,11\}\) \\
Inner iterations \(\mathcal{L}\)
    & \(\{1,2,4,6,8,10\}\) \\
Step size \(\varepsilon\)
    & \(\{0.01,0.025,0.04,0.05,0.075,0.1\}\) \\
Learning rate
    & \(\{5\times10^{-4},10^{-3},2\times10^{-3},2\times10^{-2}\}\) \\
Operator & \(\{\text{classical}, \text{normalized},\text{directional}\}\) \\
Sheaf type & \(\{\text{general}\}\) \\
Dissipation
    & \(\{\text{False},\text{True}\}\) \\
Forcing
    & \(\{\text{False},\text{True}\}\) \\
Fixed Restriction maps
    & \(\{\text{True}\}\) \\
GNN depth & \(\{0,1,2,3,4,5,6\}\) \\
GNN hidden dim & \(\{64, 128, 256, 512\}\) \\
\bottomrule
\end{tabular}
\end{table}

\paragraph{Results.}
\Cref{tab:results_hetero_complete} reports the performance on the heterophilic benchmarks. \model{} achieves highly competitive performance across the heterophilic benchmarks despite not being specifically designed for these tasks, ranking third overall based on the average rank across the five datasets. Notably, \model{} consistently improves over SONAR across all five datasets, showing that extending oscillatory propagation to the sheaf setting is also beneficial beyond explicitly long-range tasks. Similarly, \model{} substantially outperforms NSD on every dataset, indicating that richer sheaf transport alone is not sufficient to match the performance obtained with the proposed oscillatory dynamics. Compared with more recent sheaf architectures, \model{} remains competitive with BuNN and CSNN. Overall, these results show that \model{} maintains strong performance on heterophilic graphs while improving over both its closest oscillatory and standard sheaf counterparts.

\begin{table}[p]
\setlength{\tabcolsep}{3pt}
\centering
\caption{Mean test set score and std on heterophilic datasets. The higher, the better. 
\one{First}, \two{second}, and \three{third}  best results for each task are color-coded. Baseline results are reported from \cite{finkelshtein2024cooperative, 
platonov2023a, muller2024attending, luan2024heterophilicgraphlearninghandbook, ribeiro2026cooperative}. 
}
\vspace{1mm}
\label{tab:results_hetero_complete}
\small
\begin{tabular}{lccccc}
\hline\toprule
\multirow{2}{*}{\textbf{Model}} & \textbf{Roman-empire} & \textbf{Amazon-ratings} & \textbf{Minesweeper} & \textbf{Tolokers} & \textbf{Questions}\\
& \scriptsize{Acc $\uparrow$} & \scriptsize{Acc $\uparrow$} & \scriptsize{AUC $\uparrow$} & \scriptsize{AUC $\uparrow$} & \scriptsize{AUC $\uparrow$}\\
\midrule
\multicolumn{4}{l}{\textbf{\cite{luan2024heterophilicgraphlearninghandbook}}} \\
$\,$ MLP-2 & 66.04$_{\pm0.71}$ & 49.55$_{\pm0.81}$ & 50.92$_{\pm1.25}$  & 74.58$_{\pm0.75}$ & 69.97$_{\pm1.16}$\\
$\,$ SGC-1 & 44.60$_{\pm0.52}$ & 40.69$_{\pm0.42}$ & 82.04$_{\pm0.77}$  & 73.80$_{\pm1.35}$ & 71.06$_{\pm0.92}$\\
$\,$ MLP-1 & 64.12$_{\pm0.61}$ &  38.60$_{\pm0.41}$ &  50.59$_{\pm0.83}$  & 71.89$_{\pm0.82}$ & 70.33$_{\pm0.96}$\\
\midrule
\multicolumn{4}{l}{\textbf{Graph-agnostic}} \\
$\,$ ResNet      & 65.88$_{\pm0.38}$ & 45.90$_{\pm0.52}$ & 50.89$_{\pm1.39}$ & 72.95$_{\pm1.06}$ & 70.34$_{\pm0.76}$\\
$\,$ ResNet+SGC  & 73.90$_{\pm0.51}$ & 50.66$_{\pm0.48}$ & 70.88$_{\pm0.90}$ & 80.70$_{\pm0.97}$ & 75.81$_{\pm0.96}$\\
$\,$ ResNet+adj  & 52.25$_{\pm0.40}$ & 51.83$_{\pm0.57}$ & 50.42$_{\pm0.83}$ & 78.78$_{\pm1.11}$ & 75.77$_{\pm1.24}$\\
\midrule
\textbf{MPNNs} \\
$\,$ GAT                        & 80.87$_{\pm0.30}$ & 49.09$_{\pm0.63}$ & 92.01$_{\pm0.68}$ & 83.70$_{\pm0.47}$ & 77.43$_{\pm1.20}$\\
$\,$ GAT-sep                    & 88.75$_{\pm0.41}$ & 52.70$_{\pm0.62}$ & 93.91$_{\pm0.35}$ & 83.78$_{\pm0.43}$ & 76.79$_{\pm0.71}$\\
$\,$ GAT (LapPE)                & 84.80$_{\pm0.46}$ & 44.90$_{\pm0.73}$ & 93.50$_{\pm0.54}$ & 84.99$_{\pm0.54}$ & 76.55$_{\pm0.84}$\\
$\,$ GAT (RWSE)                 & 86.62$_{\pm0.53}$ & 48.58$_{\pm0.41}$ & 92.53$_{\pm0.65}$ & 85.02$_{\pm0.67}$ & 77.83$_{\pm1.22}$\\
$\,$ GAT (DEG)                  & 85.51$_{\pm0.56}$ & 51.65$_{\pm0.60}$ & 93.04$_{\pm0.62}$ & 84.22$_{\pm0.81}$ & 77.10$_{\pm1.23}$\\
$\,$ GatedGCN                   & 74.46$_{\pm0.54}$ & 43.00$_{\pm0.32}$ &  87.54$_{\pm1.22}$ &  77.31$_{\pm1.14}$ & 76.61$_{\pm{1.13}}$\\
$\,$ GCN                        & 73.69$_{\pm0.74}$ & 48.70$_{\pm0.63}$ & 89.75$_{\pm0.52}$ & 83.64$_{\pm0.67}$ & 76.09$_{\pm1.27}$\\
$\,$ GCN (LapPE)                & 83.37$_{\pm0.55}$ & 44.35$_{\pm0.36}$ & 94.26$_{\pm0.49}$ & 84.95$_{\pm0.78}$ & 77.79$_{\pm1.34}$\\
$\,$ GCN (RWSE)                 & 84.84$_{\pm0.55}$ & 46.40$_{\pm0.55}$ & 93.84$_{\pm0.48}$ & \three{85.11$_{\pm0.77}$} & 77.81$_{\pm1.40}$\\
$\,$ GCN (DEG)                  & 84.21$_{\pm0.47}$ & 50.01$_{\pm0.69}$ & 94.14$_{\pm0.50}$ & 82.51$_{\pm0.83}$ & 76.96$_{\pm1.21}$\\
$\,$ CoGNN($\Sigma$, $\Sigma$)  & 91.57$_{\pm0.32}$ & 51.28$_{\pm0.56}$ &  95.09$_{\pm1.18}$ &  83.36$_{\pm0.89}$ & \one{80.02$_{\pm0.86}$} \\
$\,$ CoGNN($\mu$, $\mu$)        & 91.37$_{\pm0.35}$ &  \one{54.17$_{\pm0.37}$} &  97.31$_{\pm0.41}$ &  84.45$_{\pm1.17}$ & 76.54$_{\pm0.95}$ \\
$\,$ GraphSAGE                  & 85.74$_{\pm0.67}$ & \three{53.63$_{\pm0.39}$} & 93.51$_{\pm0.57}$ & 82.43$_{\pm0.44}$ & 76.44$_{\pm0.62}$\\
$\,$ SONAR                      & 89.82$_{\pm 0.57}$ & 52.22$_{\pm 0.14}$ & 96.29$_{\pm 0.73}$ & 83.57$_{\pm 1.44}$ & 74.96$_{\pm 1.10}$\\
\midrule
\textbf{Graph Transformers} \\
$\,$ Exphormer   & 89.03$_{\pm0.37}$  &  53.51$_{\pm0.46}$  &  90.74$_{\pm0.53}$  &  83.77$_{\pm0.78}$ & 73.94$_{\pm1.06}$\\
$\,$ NAGphormer  & 74.34$_{\pm0.77}$  &  51.26$_{\pm0.72}$  &  84.19$_{\pm0.66}$  &  78.32$_{\pm0.95}$ & 68.17$_{\pm1.53}$\\
$\,$ GOAT        & 71.59$_{\pm1.25}$  &  44.61$_{\pm0.50}$  &  81.09$_{\pm1.02}$  &  83.11$_{\pm1.04}$ & 75.76$_{\pm1.66}$\\
$\,$ GPS         & 82.00$_{\pm0.61}$  &  53.10$_{\pm0.42}$  &  90.63$_{\pm0.67}$  &  83.71$_{\pm0.48}$ & 71.73$_{\pm1.47}$\\
$\,$ GPS\textsubscript{GCN+Performer} (LapPE) & 83.96$_{\pm0.53}$ & 48.20$_{\pm0.67}$ & 93.85$_{\pm0.41}$ & 84.72$_{\pm0.77}$ & 77.85$_{\pm1.25}$\\
$\,$ GPS\textsubscript{GCN+Performer} (RWSE) & 84.72$_{\pm0.65}$ & 48.08$_{\pm0.85}$ & 92.88$_{\pm0.50}$ & 84.81$_{\pm0.86}$ & 76.45$_{\pm1.51}$\\
$\,$ GPS\textsubscript{GCN+Performer} (DEG) & 83.38$_{\pm0.68}$ & 48.93$_{\pm0.47}$ & 93.60$_{\pm0.47}$ & 80.49$_{\pm0.97}$ & 74.24$_{\pm1.18}$\\
$\,$ GPS\textsubscript{GAT+Performer} (LapPE) & 85.93$_{\pm0.52}$ & 48.86$_{\pm0.38}$ & 92.62$_{\pm0.79}$ & 84.62$_{\pm0.54}$ & 76.71$_{\pm0.98}$\\
$\,$ GPS\textsubscript{GAT+Performer} (RWSE) & 87.04$_{\pm0.58}$ & 49.92$_{\pm0.68}$ & 91.08$_{\pm0.58}$ & 84.38$_{\pm0.91}$ & 77.14$_{\pm1.49}$\\
$\,$ GPS\textsubscript{GAT+Performer} (DEG) & 85.54$_{\pm0.58}$ & 51.03$_{\pm0.60}$ & 91.52$_{\pm0.46}$ & 82.45$_{\pm0.89}$ & 76.51$_{\pm1.19}$\\
$\,$ GT          & 86.51$_{\pm0.73}$ & 51.17$_{\pm0.66}$ & 91.85$_{\pm0.76}$ & 83.23$_{\pm0.64}$ & 77.95$_{\pm0.68}$\\
$\,$ GT-sep      & 87.32$_{\pm0.39}$ & 52.18$_{\pm0.80}$ & 92.29$_{\pm0.47}$ & 82.52$_{\pm0.92}$ & 78.05$_{\pm0.93}$\\
\midrule
\multicolumn{4}{l}{\textbf{Heterophily-Designated GNNs}} \\
$\,$ CPGNN       & 63.96$_{\pm0.62}$ & 39.79$_{\pm0.77}$ & 52.03$_{\pm5.46}$ & 73.36$_{\pm1.01}$ & 65.96$_{\pm1.95}$\\
$\,$ FAGCN       & 65.22$_{\pm0.56}$ & 44.12$_{\pm0.30}$ & 88.17$_{\pm0.73}$ & 77.75$_{\pm1.05}$ & 77.24$_{\pm1.26}$\\
$\,$ FSGNN       & 79.92$_{\pm0.56}$ & 52.74$_{\pm0.83}$ & 90.08$_{\pm0.70}$ & 82.76$_{\pm0.61}$ & \three{78.86$_{\pm0.92}$}\\
$\,$ GBK-GNN     & 74.57$_{\pm0.47}$ & 45.98$_{\pm0.71}$ & 90.85$_{\pm0.58}$ & 81.01$_{\pm0.67}$ & 74.47$_{\pm0.86}$\\
$\,$ GloGNN      & 59.63$_{\pm0.69}$ & 36.89$_{\pm0.14}$ & 51.08$_{\pm1.23}$ & 73.39$_{\pm1.17}$ & 65.74$_{\pm1.19}$\\
$\,$ GPR-GNN     & 64.85$_{\pm0.27}$ & 44.88$_{\pm0.34}$ & 86.24$_{\pm0.61}$ & 72.94$_{\pm0.97}$ & 55.48$_{\pm0.91}$\\
$\,$ H2GCN       & 60.11$_{\pm0.52}$ & 36.47$_{\pm0.23}$ & 89.71$_{\pm0.31}$ & 73.35$_{\pm1.01}$ & 63.59$_{\pm1.46}$\\
$\,$ JacobiConv  & 71.14$_{\pm0.42}$ & 43.55$_{\pm0.48}$ & 89.66$_{\pm0.40}$ & 68.66$_{\pm0.65}$ & 73.88$_{\pm1.16}$\\
\midrule
\multicolumn{6}{l}{\textbf{Sheaf Neural Networks}} \\
$\,$ NSD         & $80.41_{\pm0.72}$ & $42.76_{\pm0.54}$ & $92.15_{\pm0.84}$ & $78.83_{\pm0.76}$ & $69.69_{\pm1.46}$ \\
$\,$ BuNN        & \three{$91.75_{\pm0.39}$} & \two{$53.74_{\pm0.51}$} & \two{$98.99_{\pm0.16}$} & $84.78_{\pm0.80}$ & $78.75_{\pm1.09}$ \\
$\,$ CSNN        & \one{$92.63_{\pm0.50}$} & $52.07_{\pm1.00}$ & \one{$99.07_{\pm0.25}$} & \one{$85.45_{\pm0.53}$} & \two{$79.31_{\pm1.22}$} \\
\midrule
\multicolumn{6}{l}{\textbf{Ours}} \\
$\,$ \model{} & \two{$91.86_{\pm0.31}$}  & $ 52.76_{\pm0.48}$  & \three{$98.95_{\pm0.54}$}  &  \two{$85.21 _{\pm 0.9}$}  & $ 76.14 _{\pm 1.30} $  \\

\bottomrule\hline      
\end{tabular}
\end{table}

\section{Time and Memory Complexity Analysis}
\label{app:complexity}
In this section, we provide both a theoretical and experimental analysis of the time and space complexity of \model{}. Let \(n=|\gV|\) and \(m=|\gE|\), let \(d\) denote the stalk dimension, \(c\) the number of feature channels, \(\mathcal{L}\) the number of internal propagation steps per block, and \(B\) the number of blocks. Since \model{} operates through graph-local propagation, its computational cost is determined by the stalk dimension, the number of feature channels, and the parameterization of the restriction maps.

\noindent\textbf{Time Complexity.}
Consider a single propagation step. The state and
velocity \(\mV\) each contain \(ndc\) scalar entries. The stalk transformation \((\mI_n\otimes\mW_1)\mX\) costs \(\gO(nd^2c)\), while the channel mixer \(\mW_2\) costs \(\gO(ndc^2)\). The cost of applying the sheaf operator depends on the restriction-map family (see Appendix~\ref{app:restriction_parameterizations}). Writing \(q(d)\) for the cost of applying one restriction map to one feature channel, we have:
\begin{equation}
    q(d)
    = \begin{cases}
        d, & \text{diagonal},\\
        d^2, & \text{general or orthogonal},\\
        rd, & \text{rank-}r\text{ low-rank},
    \end{cases}
    \label{eq:restriction-cost}
\end{equation}
where the low-rank complexity assumes that the restriction maps are applied in factorized form. A sparse application of the sheaf operator therefore costs \(\gO(mq(d)c)\). Since the restriction maps may be recomputed from the current state at each propagation step, we denote the corresponding learner cost by \(\gC_{\mathrm{map}}\). The velocity update uses the sheaf propagation term already
counted above. The remaining updates of \(\mV\) and \(\mX\) involve only elementwise scaling and addition, contributing \(\gO(ndc)\) operations per step. This contribution is absorbed by the node-wise transformation costs, so the combined cost of one \model{} propagation step is:
\begin{equation}
    T_{\mathrm{\model{}}}
    = 
    \gO\!\left(
        \gC_{\mathrm{map}}
        + mq(d)c
        + nd^2c
        + ndc^2
    \right).
    \label{eq:sheafsonar-complexity}
\end{equation}
Across \(B\) blocks, each with \(\gL\) propagation steps, the total propagation cost is:
\begin{equation*}
    T_{\mathrm{\model{}}}^{\mathrm{net}}
    = \begin{cases}
        \gO\!\left(
            B\gC_{\mathrm{map}} + B\gL( mq(d)c + nd^2c + ndc^2)
        \right),
        & \text{when maps fixed within each block},\\[2mm]
        \gO\!\left(
            B\gL\left(\gC_{\mathrm{map}} + mq(d)c + nd^2c + ndc^2 \right)
        \right),
        & \text{when maps recomputed at each step}.
    \end{cases}
\end{equation*}

In particular, for fixed stalk dimension \(d\), feature dimension \(c\), and restriction-map parameterization, \model{} retains linear scaling with the graph size \(n+m\). Thus, its matrix-valued sheaf transport preserves the locality and sparse graph scaling of message-passing architectures \citep{MPNN}. In contrast, Graph Transformers based on dense global self-attention typically incur quadratic complexity in the number of nodes, \(\gO(n^2)\), due to all-pairs interactions. Thus, despite its richer local transport, \model{} preserves the sparse scaling characteristic of MPNNs and avoids the quadratic dependence on \(n\) of dense Graph Transformers.

\noindent\textbf{Memory Complexity.}
The instantaneous state \(\mX\) and velocity \(\mV\) together contain \(2ndc\) scalar entries and therefore require \(\gO(ndc)\) memory. The learned restriction maps additionally require \(\gO(ms(d))\) storage, where:
\begin{equation}
    s(d)
    =
    \begin{cases}
        d, & \text{diagonal},\\
        d^2, & \text{general or orthogonal},\\
        rd, & \text{rank-}r\text{ low-rank},
    \end{cases}
\end{equation}
up to constant factors associated with the chosen parameterization. Thus, excluding trainable parameters and common graph storage, inference requires:
\begin{equation}
    M_{\mathrm{\model{}}}^{\mathrm{infer}}
    = 
    \gO\!\left(ndc+ms(d)\right).
    \label{eq:inference-memory}
\end{equation}
During training, backpropagation through the unrolled second-order dynamics requires retaining intermediate activations. Without activation checkpointing, this gives the leading-order bound:
\begin{equation}
    M_{\mathrm{\model{}}}^{\mathrm{train}}
    =
    \gO\!\left(
        B\gL\left(ndc+ms(d)\right)
    \right).
    \label{eq:training-memory}
\end{equation}

For fixed \(d\) and restriction-map parameterization, the memory requirements of \model{} also scale linearly with \(n+m\), as in MPNNs. This contrasts with dense Graph Transformers, whose attention matrices require \(\gO(n^2)\) memory per layer. Hence, both the computational and memory complexity of \model{} remain compatible with sparse graph processing, with the additional cost arising from the stalk-valued representation and matrix-valued restriction maps rather than from dense pairwise interactions.

\noindent\textbf{Matched-parameter Empirical Cost.}
To complement the asymptotic analysis, we measure wall-clock time and peak GPU memory on the Barbell benchmark under a controlled matched-parameter setting. We compare \model{} with SONAR \citep{gravina_sonar}, a message-passing architecture based on oscillatory dynamics, to assess the practical computational overhead introduced by the sheaf-based components of \model{}. All models use \(B=4\) blocks, \(\gL=12\) internal propagation steps per block, and approximately \(100\)k trainable parameters. Measurements are obtained with batch size one (to avoid confounding effects from batch-level computational optimizations) on an RTX~5080, with one exclusive GPU process at a time. 
\Cref{tab:sonar-complexity} reports results on multiple  settings of the problem, \(N=10, 20, 50\), showing results for different implementations of \model{}'s restriction-map parametrization (see Appendix~\ref{app:restriction_parameterizations}). \begin{table*}[t]
\centering
\caption{\textbf{Matched-parameter computational cost on Barbell.}
All models use approximately \(100\)k trainable parameters, four blocks, and \(12\) internal propagation steps per block.
We report peak training memory, training time per epoch, and inference latency per graph for
\(N\in\{10,20,50\}\), corresponding to graphs with \(20\), \(40\), and \(100\) nodes, respectively.
Values are means over five seeds; lower is better.}
\label{tab:sonar-complexity}
\vspace{1mm}
\setlength{\tabcolsep}{3pt}
\resizebox{\textwidth}{!}{%
\begin{tabular}{l r ccc ccc ccc}
\toprule
&
& \multicolumn{3}{c}{\(\boldsymbol{N=10}\)}
& \multicolumn{3}{c}{\(\boldsymbol{N=20}\)}
& \multicolumn{3}{c}{\(\boldsymbol{N=50}\)}
\\
\cmidrule(lr){3-5}
\cmidrule(lr){6-8}
\cmidrule(lr){9-11}

\textbf{Model}
& \textbf{\#Params}
& \textbf{Mem.}
& \textbf{Train}
& \textbf{Infer.}
& \textbf{Mem.}
& \textbf{Train}
& \textbf{Infer.}
& \textbf{Mem.}
& \textbf{Train}
& \textbf{Infer.}
\\

&
& \textbf{MB}
& \textbf{ms/ep.}
& \textbf{ms/g}
& \textbf{MB}
& \textbf{ms/ep.}
& \textbf{ms/g}
& \textbf{MB}
& \textbf{ms/ep.}
& \textbf{ms/g}
\\
\midrule

SONAR
& 100,079
& 23.1
& 2400.1
& 10.245
& 35.1
& 2452.0
& 10.800
& 120.7
& 2675.2
& 11.138
\\

\midrule

Diag-\model{}
& 101,405
& 26.8
& 2186.3
& 6.716
& 50.9
& 2193.7
& 6.795
& 223.0
& 2270.4
& 6.722
\\

Ortho-\model{}
& 103,700
& 24.5
& 2229.1
& 7.849
& 41.1
& 2237.4
& 7.962
& 159.9
& 2432.3
& 7.925
\\

LowRank-\model{}
& 101,540
& 27.6
& 2464.3
& 8.371
& 53.9
& 2459.3
& 8.290
& 241.8
& 2736.4
& 8.610
\\

General-\model{}
& 104,814
& 27.4
& 2404.6
& 8.150
& 53.1
& 2447.9
& 8.138
& 236.5
& 2729.5
& 8.604
\\

\bottomrule
\end{tabular}%
}
\end{table*}
Despite the richer per-edge computation predicted by \Cref{eq:sheafsonar-complexity}, the measured runtime remains comparable to SONAR under a matched parameter budget. 
At \(N=50\), Diag-\model{} and Ortho-\model{} reduce training time by \(15.1\%\) and \(9.1\%\), respectively, relative to SONAR, while their inference latency is lower by \(39.6\%\) and \(28.8\%\). The more expressive LowRank and General variants incur only a \(2.3\%\) and \(2.0\%\) increase in training time, respectively, while still reducing inference latency by approximately \(22.7\%\). Across \(N\in\{10,20,50\}\), all four \model{} variants exhibit lower measured inference latency than SONAR. At \(N=50\), peak training memory ranges from \(159.9\) to \(241.8\) MB for the sheaf variants, compared with \(120.7\) MB for SONAR. This increase is consistent with \Cref{eq:training-memory}: \model{} maintains \(d\)-dimensional stalk states and learned edge-wise restriction maps throughout the unrolled dynamics. Importantly, this additional memory buys a strictly richer propagation mechanism while preserving sparse, linear-in-\(n+m\) graph scaling. Hence, for the small stalk dimensions used in practice, \model{} retains the computational profile of local SONAR propagation while replacing scalar communication with expressive matrix-valued transport.

\end{document}